\documentclass{article}
\usepackage{ijcai21}

\usepackage{times}
\usepackage{soul}
\usepackage{url}
\usepackage[hidelinks]{hyperref}
\usepackage[utf8]{inputenc}
\usepackage[small]{caption}
\usepackage{graphicx}
\usepackage{amsmath}
\usepackage{amsthm}
\usepackage{booktabs}
\usepackage{algorithm}
\usepackage[T1]{fontenc}    % use 8-bit T1 fonts
\usepackage{amsfonts}       % blackboard math symbols
\usepackage{nicefrac}       % compact symbols for 1/2, etc.
\usepackage{microtype}      % microtypography

\usepackage{amssymb}
\usepackage{mathtools}
\usepackage{algpseudocode} 
\usepackage{xcolor}

\DeclarePairedDelimiter\ceil{\lceil}{\rceil}

\DeclareMathOperator*{\argmax}{arg\,max}

\usepackage[symbol]{footmisc}

\usepackage[english]{babel}
\usepackage{multicol}
\usepackage{wrapfig} %MSAL
\usepackage{subcaption}
\usepackage{threeparttable}

\usepackage{stfloats}
\usepackage{adjustbox}

\newtheorem{theorem}{Theorem}[section]
\title{Generative Adversarial Neural Architecture Search}

\author{
Seyed Saeed Changiz Rezaei$^1$\footnote{Equal Contribution. Correspondence to: Di Niu <dniu@ualberta.ca>, Seyed Rezaei <seyeds.rezaei@huawei.com>.}\and
Fred X. Han$^{1*}$\and
Di Niu$^2$\and
Mohammad Salameh$^1$\and
Keith Mills$^2$\and
Shuo Lian$^3$\and
Wei Lu$^1$\And
Shangling Jui$^3$\\
\affiliations
$^1$Huawei Technologies Canada Co., Ltd.\\
$^2$Department of Electrical and Computer Engineering, University of Alberta\\
$^3$Huawei Kirin Solution, Shanghai, China\\
}
\begin{document}

\maketitle

\begin{abstract}
Despite the empirical success of neural architecture search (NAS) in deep learning applications, the optimality, reproducibility and cost of NAS schemes remain hard to assess.  In this paper, we propose Generative Adversarial NAS (GA-NAS) with theoretically provable convergence guarantees, promoting stability and reproducibility in neural architecture search. Inspired by importance sampling, GA-NAS iteratively fits a generator to previously discovered top architectures, thus increasingly focusing on important parts of a large search space. Furthermore, we propose an efficient adversarial learning approach, where the generator is trained by reinforcement learning based on rewards provided by a discriminator, thus being able to explore the search space without evaluating a large number of architectures. Extensive experiments show that GA-NAS beats the best published results under several cases on three public NAS benchmarks. In the meantime, GA-NAS can handle ad-hoc search constraints and search spaces. We show that GA-NAS can be used to improve already optimized baselines found by other NAS methods, including EfficientNet and ProxylessNAS, in terms of ImageNet accuracy or the number of parameters, in their original search space.
\end{abstract}

\section{Introduction}
%Designing an appropriate Neural Architecture was used to be performed manually by human efforts. For example, VGG, ResNet, Inception, and MobileNet are instances of the state-of-the art manually designed neural architectures for computer vision applications (see \cite{simonyan2014very}, \cite{he2016deep}, \cite{szegedy2016rethinking}, \cite{howard2017mobilenets}). However, designing a Neural Architecture manually is a daunting task which requires a large amount of expertise and trials and errors. This fact leads to a growing interest in automatic Neural Architecture Search (NAS) which achieves the state-of-the art performances on CIFAR-10, ImageNet, and COCO datasets (see \cite{pham2018ENAS}, \cite{liu2018DARTS}, \cite{cai2018proxylessnas}, \cite{tan2018mnasnet}, \cite{wu2019fbnet}, \cite{yu2020bignas}).  
%\MNasnet
%However,  manually designing a neural architecture is a daunting task that requires a large amount of expertise and trials and errors. 
Neural architecture search (NAS) improves neural network model design by replacing the manual trial-and-error process with an automatic search procedure, and has achieved state-of-the-art performance on many computer vision tasks %e.g., on CIFAR-10, ImageNet, and COCO datasets 
\cite{elsken2018neural}.  
%Although NAS approaches differ in terms of the search (sampling) strategy, evaluation strategy, as well as the search space used, they can all be understood as solving an optimization problem, where the goal is to find an architecture in a predefined space of architectures to maximize certain performance metric. 
Since the underlying search space of architectures grows exponentially as a function of the architecture size, searching for an optimum neural architecture is like looking for a needle in a haystack. 
A variety of \textit{search algorithms} have been proposed for NAS, including random search (RS) \cite{li2019RS}, differentiable architecture search (DARTS) \cite{liu2018DARTS}, Bayesian optimization (BO) \cite{kandasamy2018neural}, evolutionary algorithm (EA) \cite{dai2020fbnetv3}, and reinforcement learning (RL) \cite{pham2018ENAS}.

Despite a proliferation of NAS methods proposed, their sensitivity to random seeds and reproducibility issues concern the community \cite{li2019RS}, \cite{yu2019evaluating}, \cite{yang2019evaluation}.
Comparisons between different search algorithms, such as EA, BO, and RL, etc., are particularly hard, as there is no shared search space or experimental protocol followed by all these NAS approaches.
To promote fair comparisons among methods, multiple NAS benchmarks have recently emerged,
including NAS-Bench-101 \cite{ying2019nasbench101}, NAS-Bench-201  \cite{dong2020nasbench201}, and NAS-Bench-301 \cite{siems2020bench}, which contain collections of architectures with their associated performance. This has provided an opportunity for researchers to fairly benchmark search algorithms (regardless of the search space in which they are performed) by evaluating how many queries to architectures an algorithm needs to make in order to discover a top-ranked architecture in the benchmark set \cite{luo2020semi,siems2020bench}.
The number of queries converts to an indicator of how many architectures need be evaluated in reality, which often forms the bottleneck of NAS.

%However, \cite{yu2019evaluating} has evaluated several well-known search algorithms, including DARTS~\cite{liu2018DARTS}, ENAS~\cite{pham2018ENAS} and NAO~\cite{luo2018neural} on NAS-Bench-101, and has concluded that these search algorithms perform similarly to a random policy~\cite{yu2019evaluating}. 

%that randomly selects an architecture from the search space,}
%NAS-Bench-101~\cite{ying2019nasbench101} provides the cell design and evaluation accuracies for 423,624 unique architectures on CIFAR-10~\cite{Krizhevsky09CIFAR}. This has provided an efficient way to fairly benchmark search algorithms, since evaluating an architecture becomes querying the dataset for its accuracy, and has given rise to a recent competition among NAS algorithms that aim to find the highest performing architectures with the least number of queries for unique architectures in the dataset.  

%The subsequent work of NAO~\cite{luo2018neural}, namely SemiNAS~\cite{luo2020semi}, has reported the state-of-the-art performance on the NAS-Bench-101 dataset, which can find an architecture that ranks the 15th in 2100 queries to the dataset.

%In this paper, we revisit the NAS problem with \emph{importance sampling}, a robust method for rare event discovery with theoretical guarantees \cite{rubinstein2016simulation,rubinstein2013cross}.
In this paper, we propose Generative Adversarial NAS (GA-NAS), a provably converging and efficient  search algorithm to be used in NAS based on adversarial learning.
Our method is first inspired by the \emph{Cross Entropy} (CE) method \cite{rubinstein2013cross} in importance sampling, which iteratively retrains an architecture generator to fit to the distribution of winning architectures generated in previous iterations so that the generator will increasingly sample from more important regions in an extremely large search space. However, such a generator cannot be efficiently trained through back-propagation, as performance measurements can only be obtained for discretized architectures and thus the model is not differentiable. To overcome this issue, GA-NAS uses RL to train an architecture generator network based on RNN and GNN. Yet unlike other RL-based NAS schemes, GA-NAS does not obtain rewards by evaluating generated architectures, which is a costly procedure if a large number of architectures are to be explored. Rather, it learns a discriminator to distinguish the winning architectures from randomly generated ones in each iteration. This enables the generator to be efficiently trained based on the rewards provided by the discriminator, without many true evaluations. We further establish the convergence of GA-NAS in a finite number of steps, by connecting GA-NAS to an importance sampling method with a symmetric Jensen–Shannon (JS) divergence loss.

Extensive experiments have been performed to evaluate GA-NAS in terms of its convergence speed, reproducibility and stability in the presence of random seeds, scalability, flexibility of handling constrained search, and its ability to improve already optimized baselines. We show that GA-NAS outperforms a wide range of existing NAS algorithms, including EA, RL, BO, DARTS, etc., and state-of-the-art results reported on three representative NAS benchmark sets, including NAS-Bench-101, NAS-Bench-201, and NAS-Bench-301---it consistently finds top ranked architectures within a lower number of queries to architecture performance. We also demonstrate the flexibility of GA-NAS by showing its ability to incorporate ad-hoc hard constraints and its ability to further improve existing strong architectures found by other NAS methods.
Through experiments on ImageNet, we show that GA-NAS can enhance EfficientNet-B0 \cite{tan2019efficientnet} and ProxylessNAS \cite{cai2018proxylessnas} in their respective search spaces, resulting in architectures with higher accuracy and/or smaller model sizes.
%\red{In addition, GA-NAS can directly find approximately a quarter of all Pareto Front cells in Nas-Bench-101 after querying 2,869 architectures}.  

%the unsupervised nature of the discriminator component of our search algorithm provides the flexibility to incorporate multiple desired objectives such as accuracy, latency, and model parameter size during the search phase. 

%Moreover, because of the unsupervised nature of our proposed search algorithm, coming from the discriminator component of the algorithm, we can easily incorporate any number of desired objectives such as accuracy, latency, and model parameter size during the search phase. \red{report the training constrained experiments here.}

\section{Related Work}
\label{sec:related}
A typical NAS method includes a \emph{search phase} and an \emph{evaluation phase}. 
This paper is concerned with the search phase, of which the most important performance criteria are robustness, reproducibility and search cost.
DARTS \cite{liu2018DARTS} has given rise to numerous optimization schemes for NAS \cite{xie2018SNAS}.  %%%\cite{li2020geometry, xie2018SNAS,chen2019progressive,xu2020pcdarts,chen2020stabilizing,li2020geometry}. 
While the objectives of these algorithms may vary, they all operate in the same or similar search space. 
However, \cite{yu2019evaluating} demonstrates that DARTS performs similarly to a random search and its results heavily dependent on the initial seed. 
% Furthermore, DARTS is criticized for converging to smooth loss landscapes which may not generalize well \cite{chen2020stabilizing}, and having skip connections dominating other operations \cite{ZhouXSH20}.
In contrast, GA-NAS has a convergence guarantee under certain assumptions and its results are reproducible.
%%%In contrast, GA-NAS has a convergence guarantee under certain assumptions and its results are reproducible and are not sensitive to initial seeds.
%%%However, \cite{yu2019evaluating} demonstrates that DARTS performs similarly to a random search and its search results heavily dependent on the initial random seed. Furthermore, DARTS is criticized for converging to architectures with smooth loss landscapes which may not generalize well \cite{shu2019understanding,chen2020stabilizing}, and having skip connections dominating other operations and leading to performance degradation \cite{ZhouXSH20}.

NAS-Bench-301 \cite{siems2020bench} provides a formal benchmark for all $10^{18}$ architectures in the DARTS search space.
Preceding NAS-Bench-301 are NAS-Bench-101 \cite{ying2019nasbench101} and NAS-Bench-201 \cite{dong2020nasbench201}. 
%%%Both of these benchmarks provide performance metrics in a tabular format and perform a fully exhaustive evaluation across all architectures in their search spaces. 
%%%GA-NAS can find high-performing architectures in all three benchmarks and proves to be a highly robust algorithm that is not sensitive to search spaces.
Besides the cell-based searches, GA-NAS also applies to macro-search \cite{cai2018proxylessnas,tan2018mnasnet}, which searches for an ordering of a predefined set of blocks. 
%%%We show that GA-NAS can be used to improve EfficientNet \cite{tan2019efficientnet} by generating better ordering of the same set of MBConv blocks.

On the other hand, several RL-based NAS methods have been proposed. 
ENAS \cite{pham2018ENAS} is the first Reinforcement Learning scheme in weight-sharing NAS.
TuNAS \cite{bender2020can} shows that guided policies decisively exceed the performance of random search on vast search spaces. 
%%%AlphaX \cite{wang2018neural} uses Monte Carlo Tree Search to balance exploration with exploitation during the search. 
In comparison, GA-NAS proves to be a highly efficient RL solution to NAS, since the rewards used to train the actor come from the discriminator instead of from costly evaluations. 
%%%Our ablation studies show that the use of a discriminator can lower the number of architecture evaluations tremendously, which is typically a bottleneck in any NAS method.

%%%Hardware-friendly NAS algorithms may take constraints such as model size, FLOPS, and inference time into account \cite{cai2018proxylessnas,tan2018mnasnet,wu2019fbnet,chen2020multi,yu2020bignas,LinCLCG020}, usually by introducing regularizers into the loss functions \cite{wu2019fbnet,cai2018proxylessnas}. 
Hardware-friendly NAS algorithms take constraints such as model size, FLOPS, and inference time into account \cite{cai2018proxylessnas,tan2018mnasnet}, usually by introducing regularizers into the loss functions. 
Contrary to these methods, GA-NAS can support ad-hoc search tasks, by enforcing customized hard constraints in importance sampling instead of resorting to approximate penalty terms.

\section{The Proposed Method}
\label{sec:algorithm_generator}
In this section, we present Generative Adversarial NAS (GA-NAS) as a search algorithm to discover top architectures in an extremely large search space.
 GA-NAS is theoretically inspired by the importance sampling approach and implemented by a generative adversarial learning framework to promote efficiency.

We can view a NAS problem as a combinatorial optimization problem. For example, suppose that $x$ is a Directed Acyclic Graph (DAG) connecting a certain number of operations, each chosen from a predefined operation set. Let $S(x)$ be a real-valued function representing the performance, e.g., accuracy, of $x$. 
In NAS, we optimize $S(x)$ subject to $x\in \mathcal X$, where $\mathcal X$ denotes the underlying search space of neural architectures.

One approach to solving a combinatorial optimization problem, especially the NP-hard ones, is to view the problem in the framework of \textit{importance sampling} and \emph{rare event simulation} \cite{rubinstein2013cross}. In this approach we consider a family of probability densities $\{p(.;\theta)\}_{ \theta \in \Theta}$ on the set $\mathcal X$, with the goal of finding a density $p(.;\theta^*)$ that assigns higher probabilities to  optimal solutions to the problem. Then, with high probability, the sampled solution from the density $p(.;\theta^*)$ will be an optimal solution. This approach for importance sampling is called the \emph{Cross-Entropy} method. %and is explained in Appendix. %\ref{sec:crossEntropy}.

\subsection{Generative Adversarial NAS}

\begin{algorithm}[t]
	\caption{GA-NAS Algorithm} \label{alg:alg3}
	\begin{algorithmic}[1]
	    \State \textbf{Input:} An initial set of architectures $\mathcal X_0$; Discriminator $D$; Generator $G(x;\theta_0)$; A positive integer $k$;
	    \For {$t=1,2,\ldots, T$}
	        \State $\mathcal T\leftarrow$ top $k$ \text{ architectures  of}~$\bigcup_{i=0}^{t-1}\mathcal X_i$ according to the performance evaluator $S(.)$.
	        \State Train $G$ and $D$ to obtain their parameters 
	        \begin{equation}
	        \begin{multlined}[c]
	        (\theta_t, \phi_t)\longleftarrow\arg\min_{\theta_t}\max_{\phi_t} V\left(G_{\theta_t}, D_{\phi_t}\right) =\\ \mathbb{E}_{x\sim p_{\mathcal T}} \left[\log D_{\phi_t}(x)\right] + \mathbb{E}_{x\sim p(x;\theta_t)}\left[\log \left(1 - D_{\phi_t}(x)\right)\right].\nonumber
	        \end{multlined}
	        \end{equation} \label{step:minimax}
	        \State Let $G(x;\theta_t)$ generate a new set of architectures $\mathcal X_t$.
	        \State Evaluate the performance $S(x)$ of every $x \in \mathcal X_t$.
	   \EndFor
	\end{algorithmic} 
\end{algorithm}

 GA-NAS is presented in Algorithm~\ref{alg:alg3}, where
 a discriminator $D$ and a generator $G$ are learned over iterations.
 In each iteration of GA-NAS, the generator $G$ will generate a new set of architectures $\mathcal X_t$, which gets combined with previously generated architectures. 
 A set of top performing architectures $\mathcal T$, which we also call \emph{true} set of architectures, is updated by selecting top $k$ from already generated architectures and gets improved over time.

 Following the GAN framework, originally proposed by \cite{goodfellow2014generative}, the discriminator $D$ and the generator $G$ are trained by playing a two-player minimax game whose corresponding parameters $\phi_t$ and $\theta_t$ are optimized by Step \ref{step:minimax} of Algorithm~\ref{alg:alg3}.
 After Step \ref{step:minimax}, $G$ learns to generate architectures to fit to the distribution of $\mathcal T$, the true set of architectures.
 
Specifically, after Step \ref{step:minimax} in the $t$-th iteration, the generator $G(.;\theta_t)$ parameterized by $\theta_t$ learns to generate architectures from a family of probability densities $\{p(.;\theta)\}_{ \theta \in \Theta}$ such that $p(.;\theta_t)$ will approximate $p_{\mathcal T}$, the architecture distribution in true set $\mathcal T$, while $\mathcal T$ also gets improved for the next iteration by reselecting top $k$ from generated architectures.

%the best $k$ architectures are considered to serve as the true set of architectures $\mathcal T$ for the next iteration. 

 %Based on theory of importance sampling \cite{rubinstein2013cross,rubinstein2016simulation},
 We have the following theorem which provides a theoretical convergence guarantee for GA-NAS under mild conditions. %its proof presented in the Appendix.

\begin{theorem}\label{thm:main_theorem}
Let $\alpha > 0$ be any real number that indicates the performance target, such that $\max_{x\in\mathcal X}~S(x) \geq \alpha$. Define $\gamma_t$ as a level parameter in iteration $t$ ($t=0,\ldots,T$),
such that the following holds for all the architectures $x_t\in\mathcal X_t$  generated by the generator $G(.;\theta_t)$  in the $t$-th iteration:
\[
S(x_t) \geq \gamma_t,\quad\forall x_t\in\mathcal X_t.
\]
Choose $k$, $|\mathcal X_t|$, and the $\gamma_t$ defined above such that
$\gamma_0 < \alpha$, and
\begin{equation}
    %S(x) > \gamma_0,\quad\forall~x\in\mathcal X_0,~\textit{for}~
    \gamma_t \geq  \min(\alpha, \gamma_{t-1} + \delta),~\textit{for some}~\delta > 0,~\forall~ t\in\{1,\ldots,T\}.\nonumber
\end{equation}
Then, GA-NAS Algorithm can find an architecture $x$ with $S(x) \geq \alpha$ in a finite number of iterations $T$.
\end{theorem}
%\begin{proof}
\proof 
Refer to Appendix %\footnote{See \url{https://arxiv.org/abs/2105.09356} for the Appendix.} 
for a proof of the theorem.%~\ref{sec:proof}.
%\qedsymbol
%\end{proof}

%{\bf Remark.}
\paragraph{Remark.} % Only 1 paragraph
This theorem indicates that as long as there exists an architecture in $\mathcal X$ with performance above $\alpha$, GA-NAS is guaranteed to find an architecture with $S(x)\geq\alpha$ in a finite number of iterations. 
From \cite{goodfellow2014generative}, the minimax game of Step \ref{step:minimax} of Algorithm~\ref{alg:alg3} is equivalent to minimizing the Jensen-Shannon (JS)-divergence between the distribution $p_{\mathcal T}$ of the currently top performing architectures $\mathcal T$ and the distribution $p(x;\theta_t)$ of the newly generated architectures. 
Therefore, our proof involves replacing the arguments around the Cross-Entropy framework in importance sampling \cite{rubinstein2013cross} with minimizing the JS divergence, as shown in the Appendix.

% Furthermore, to the best of our knowledge, we are the first who provide a convergence guarantee for a NAS algorithm. 
% In fact, Step \ref{step:minimax} of Algorithm~\ref{alg:alg3} can be implemented by alternately training the discriminator $D$ between true data $\mathcal T$ and the generated data $\mathcal{F}$ and training a generator $G$ to minimize the probability that the discriminator distinguishes generated data $\mathcal{F}$ from truth data $\mathcal T$ as in \cite{goodfellow2014generative}. This leads to the minimization of the JS-divergence between the distribution of truth data and that of the generator model (see Theorem 1 in \cite{goodfellow2014generative}). 

%We follow the GAN framework of \cite{goodfellow2014generative}, where the algorithm iteratively alternates between lines 5 and 6 multiple times to push $D$ and $G$ to reach the saddle point of a minimax optimization problem.  In each iteration, the JS-divergence between the distribution of top architectures $\mathcal T$  and the generator distribution $G(x;\theta_t)$ is minimized. If $G(x;\theta_t)$ corresponds to $p(x;\theta_t)$ in Algorithm~\ref{alg:alg2_modified}, and the distribution of architectures in $\mathcal T$ corresponds to $q(x,\gamma_{t-1})$, i.e., the distribution of top architectures in prior iteration, GA-NAS algorithm is essentially an implementation of the JS-divergence rare-event estimation  with the JS minimization step in Algorithm~\ref{alg:alg2_modified} replaced by alternately training $D$ and $G$.
%\subsection{Models and Training Procedure}
\subsection{Models}

We now describe different components in GA-NAS. %, followed by their training procedures.
Although GA-NAS can operate on any search space, %Two typical setups include a micro search space where we search for cells which are further stacked to form a neural network, and a macro search space where we search for the combination and ordering of a given set of blocks to form a neural network.
here we describe the implementation of the Discriminator and Generator in the context of cell search.  We evaluate GA-NAS for both cell search and macro search in  experiments. 
% which is similar to ENAS~\cite{pham2018ENAS} and DARTS~\cite{liu2018DARTS}.
A \emph{cell} architecture $\mathcal C$ is a Directed Acyclic Graph (DAG) consisting of multiple nodes and directed edges. Each intermediate node represents an operator, such as convolution or pooling, from a predefined set of operators. 
Each directed edge represents the information flow between nodes. We assume that a cell has at least one input node and only one output node. %and variable numbers of searchable intermediate nodes and edges.

%\textbf{Architecture Generator.} 
\paragraph{Architecture Generator.} % Only 1 paragraph
%%%The architectures can be generated in the discrete search space of DAGs, as in ENAS \cite{pham2018ENAS}, or in a continuous search space as in NAO \cite{luo2018neural}, D-VAE \cite{zhang2019d} and SemiNAS \cite{luo2020semi}. 
%%%The architectures can be generated in the discrete search space of DAGs, as in ENAS \cite{pham2018ENAS}, or in a continuous search space as in SemiNAS \cite{luo2020semi}. 
Here we generate architectures in the discrete search space of DAGs. Particularly, we generate architectures in an autoregressive fashion, which is a frequent technique in neural architecture generation such as in ENAS. 
At each time step $t$, given a partial cell architecture ${\mathcal C}_t$ generated by the previous time steps,  GA-NAS uses an encoder-decoder architecture to decide what new operation to insert and which previous nodes it should be connected to.
The encoder is a multi-layer $k$-GNN \cite{morris2019weisfeiler}.
The decoder consists of an MLP that outputs the operator probability distribution and a Gated Recurrent Unit (GRU) %\cite{chung2014empirical}
that recursively determines the edge connections to previous nodes.

%\textbf{Pairwise Architecture Discriminator.}
\paragraph{Pairwise Architecture Discrimator.} % Only 1 paragraph
In Algorithm \ref{alg:alg3}, $\mathcal T$ contains a limited number of architectures. 
To facilitate a more efficient use of $\mathcal T$, we adopt a relativistic discriminator \cite{jolicoeur2018relativistic} $D$ that follows a Siamese scheme. %\cite{koch2015siamese} scheme.
$D$ takes in a pair of architectures where one is from $\mathcal T$, and determines whether the second one is from the same distribution.
The discriminator is implemented by encoding both cells in the pair with a shared $k$-GNN followed by an MLP classifier with details provided in the Appendix.

\begin{figure}[t]
	\centering
	\includegraphics[width=3.2in]{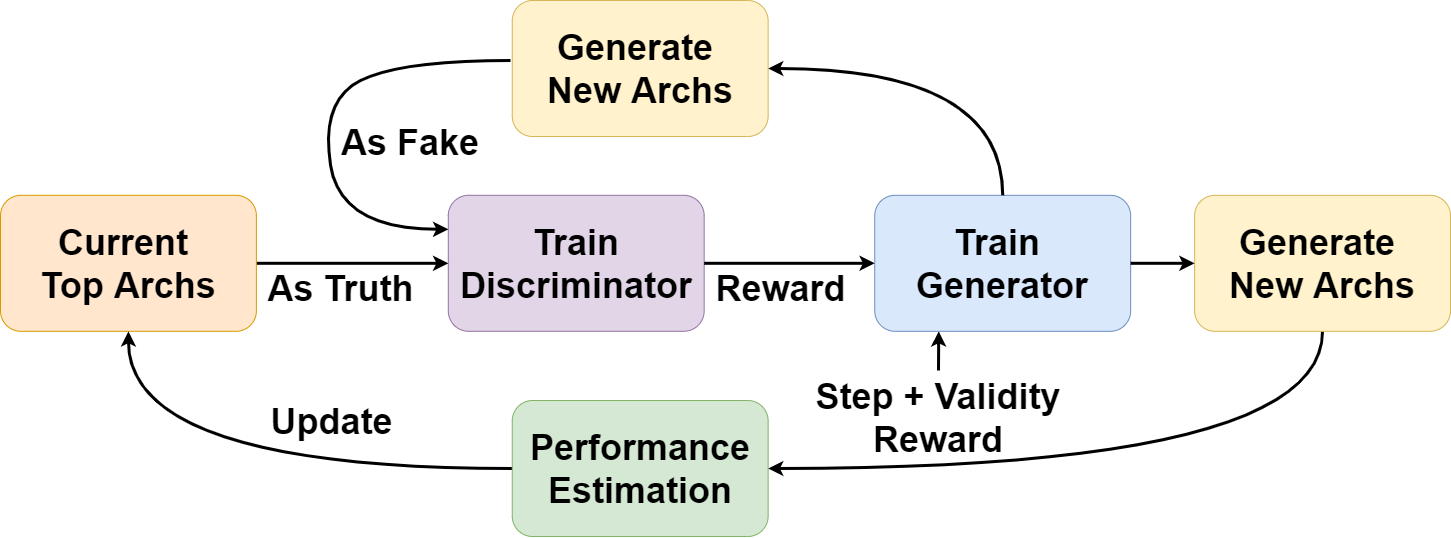} % GANAS.png for arXiv
	\caption{The flow of the proposed GA-NAS algorithm.}
	\label{fig:ga_nas_flow}
\end{figure}

%\textbf{Training Procedure.}
%\paragraph{Training Procedure.} % Since this is 2 paragraphs, changing to subsection
\subsection{Training Procedure}
Generally speaking, in order to solve the minimax problem in Step \ref{step:minimax} of Algorithm \ref{alg:alg3}, one can follow the minibatch stochastic gradient descent (SGD) training procedure for training GANs as originally proposed in \cite{goodfellow2014generative}. However, this SGD approach is not applicable to our case, since the architectures (DAGs) generated by $G$ are discrete samples, and their performance signals cannot be directly back-propagated to update $\theta_t$, the parameters of $G$. Therefore, to approximate the JS-divergence minimization between the distribution of top $k$ architectures $\mathcal T$, i.e., $p_{\mathcal T}$, and the generator distribution $p(x;\theta_t)$ in iteration $t$, we replace Step \ref{step:minimax} of Algorithm~\ref{alg:alg3} by alternately training $D$ and $G$ using a procedure outlined in Algorithm~\ref{alg:training}.  
Figure~\ref{fig:ga_nas_flow} illustrates the overall training flow of GA-NAS.
%which are pushed to reach the saddle point of a minimax optimization problem according to the original
%If $G(x;\theta_t)$ corresponds to $p(x;\theta_t)$, and the distribution of architectures in $\mathcal T$, i.e., the distribution of top architectures in prior iteration, corresponds to $q(x,\gamma_{t-1})$ in Algorithm~\ref{alg:alg2_modified}, GA-NAS algorithm is essentially an implementation of the JS-divergence rare-event estimation. It replaces the JS minimization (Step 3) in Algorithm~\ref{alg:alg2_modified} by alternately learning $D$ and $G$, and instead of keeping track of $\rho_t$ and $\gamma_t$, using the top $K$ architectures as positive samples in each iteration.
%Roughly speaking, $\rho_t=\frac{k}{|\bigcup_{i=0}^{t-1}\mathcal X_i|}$, and $\gamma_t$ is chosen so that the $(1-\rho_t)$-quantile is $\gamma_t$. 

We first train the GNN-based discriminator using pairs of architectures sampled from $\mathcal T$ and $\mathcal F$ based on supervised learning. 
Then, we use Reinforcement Learning (RL) to train $G$ with the reward defined in a similar way as in \cite{you2018graph}, including a step reward that reflects the validity of an architecture during each step of generation and a final reward that mainly comes from the discriminator prediction $D$.
% We associate the \textit{action} at each time step (including the operation type and its connections to previous nodes) with an immediate reward $R_{step}$ based on its validity given the search space constraints. 
When the architecture generation terminates, a \textit{final} reward $R_{final}$ penalizes the generated architecture $x$ according to the total number of violations of validity or rewards it with a score from the discriminator $D(x)$ that indicates how similar it is to the current true set $\mathcal T$. 
Both rewards together ensure that $G$ generates valid cells that are structurally similar to top cells from the previous time step. We adopt Proximal Policy Optimization (PPO), %\cite{schulman2017proximal}, 
a policy gradient algorithm with generalized advantage estimation to train the policy, which is also used for NAS in \cite{tan2018mnasnet}. 

% The proposed learning procedure has several benefits. 
% First, using the discriminator as a rewarding mechanism can significantly reduce the number of architecture evaluations compared to the conventional scheme where the fully-evaluated network accuracy is used as the reward. Our ablation study in Section \ref{sec:exp} verifies this claim. 
% Second, since the generator must sample a discrete architecture $x$ to obtain its loss on the discriminator $D(x)$, the entire generator-discriminator pipeline is not end-to-end differentiable and cannot be trained by SGD. Training $G$ with PPO solves this non-differentiability issue. 
% Third, in PPO loss, there is an entropy loss term that encourages variations in the generated actions. By tuning the multiplier for the entropy loss, we can balance exploration/exploitation, which is crucial for a large search space. Please refer to the Appendix for a detailed discussion.
%{\bf Remark.}
\paragraph{Remark.} % Only 1 paragraph
The proposed learning procedure has several benefits. 
\textit{First,} since the generator must sample a discrete architecture $x$ to obtain its discriminator prediction $D(x)$, the entire generator-discriminator pipeline is not end-to-end differentiable and cannot be trained by SGD as in the original GAN. Training $G$ with PPO solves this differentiability issue. 
\textit{Second,} in PPO there is an entropy loss term that encourages variations in the generated actions. By tuning the multiplier for the entropy loss, we can control the extent to which we explore a large search space of architectures.
\textit{Third,} using the discriminator outputs as the reward can significantly reduce the number of architecture evaluations compared to a conventional scheme where the fully evaluated network accuracy is used as the reward. Ablation studies in Section \ref{sec:exp} verifies this claim. Also refer to the Appendix for a detailed discussion.

\begin{algorithm}[t]
	\caption{Training Discriminator $D$ and the Generator $G$} \label{alg:training}
	\begin{algorithmic}[1]
		\State \textbf{Input:} Discriminator $D$; Generator $G(x;\theta_{t - 1})$; Data set $\mathcal T$.
	    \For {$ t^\prime =1,2,\ldots,~T^\prime$}
	        \State Let $G(x;\theta_{t-1})$ generate $k$ random architectures to form the set $\mathcal F$.
	        \State Train discriminator $D$ with $\mathcal{T}$ being positive samples and $\mathcal{F}$ being negative samples.
	        \State Using the output of $D(x)$ as the main reward signal, train $G(x;\theta_t)$ with Reinforcement Learning.
	        % Train generator $G(x;\theta_t)$ using the output of $D(x)$ as the loss.
	   \EndFor
	\end{algorithmic} 
\end{algorithm}
\section{Experimental Results}\label{sec:exp}
We evaluate the performance of GA-NAS as a search algorithm by comparing it with a wide range of existing NAS algorithms in terms of convergence speed and stability on NAS benchmarks. We also 
%First, we conduct experiments on several public NAS benchmarks under both non-weight-sharing and weight-sharing settings in order to fairly compare the powers of the search algorithms. 
show the capability of GA-NAS to improve a given network, including already optimized strong baselines such as EfficientNet and ProxylessNAS. %to achieve a higher accuracy and/or a lower number of parameters. %We also refer interested readers to the Appendix for the ablation studies showing the effects of different components of GA-NAS on its performance.

\subsection{Results on NAS Benchmarks with or without Weight Sharing}

To evaluate search algorithm and decouple it from the impact of search spaces, we query three NAS benchmarks: NAS-Bench-101 \cite{ying2019nasbench101}, NAS-Bench-201 \cite{dong2020nasbench201}, and NAS-Bench-301 \cite{siems2020bench}. The goal is to discover the highest ranked cell with as few queries as possible. 
The number of queries to the NAS benchmark is used as a reliable measure for search cost in recent NAS literature \cite{luo2020semi,siems2020bench}, because each query to the NAS benchmark corresponds to training and evaluating a candidate architecture from scratch, which constitutes the major bottleneck in the overall cost of NAS. By checking the rank an algorithm can reach in a given number of queries, one can also evaluate the convergence speed of a search algorithm. 
More queries made to the benchmark would indicate a longer evaluation time or higher search cost  in a real-world problem.

To further evaluate our algorithm when the true architecture accuracies are unknown, we train a weight-sharing supernet on NAS-Bench-101 and compare GA-NAS with a range of NAS schemes based on weight sharing.
\begin{table}[t]
   \centering
   \small
	\begin{tabular}{|l|c|c|} \hline
          \textbf{Algorithm}                      & \textbf{Acc (\%)} &  \textbf{\#Q}  \\ \hline
		Random Search  $^\dagger$          & 93.66             & 2000                 \\
		RE $^\dagger$                      & 93.97             & 2000                 \\
		NAO $^\dagger$                      & 93.87             & 2000                  \\
		BANANAS                             & \textbf{94.23}              & 800                   \\
		SemiNAS $^\dagger$                 & 94.09             & 2100                   \\
		SemiNAS $^\dagger$                  & 93.98             & 300                   \\
        \hline
		\textbf{GA-NAS}-setup1                           & \textbf{94.22}  & \textbf{150}          \\        
        \textbf{GA-NAS}-setup2                           & \textbf{94.23}  & \textbf{378}          \\
		\hline
	\end{tabular}
	%\vspace{-3mm}
	\caption{\label{table:nb101_best_acc} 
	The best accuracy values found by different search algorithms on NAS-Bench-101 without weight sharing. Note that 94.23\% and 94.22\% are the accuracies of the 2nd and 3rd best cells. $\dagger$:  taken from \protect\cite{luo2020semi}}   
	%\vspace{-2mm}
\end{table}
\textbf{NAS-Bench-101} is the first publicly available benchmark for evaluating NAS algorithms.
It consists of 423,624 DAG-style cell-based architectures, each trained and evaluated for 3 times. 
Metrics for each run include training time and accuracy.
Querying NAS-Bench-101 corresponds to evaluating a cell in reality. 
We provide the results on two setups. 
In the first setup, we set $|\mathcal X_0|=50$, $|\mathcal X_t|=|\mathcal X_{t-1}|+50,~ t \geq 1$, and $ k = 25$, 
In the second setup, we set $|\mathcal X_0|=100$, $|\mathcal X_t|=|\mathcal X_{t-1}|+100,~t \geq 1$, and $k=50$.  For both setups, the initial set $\mathcal X_0$ is picked to be a random set, and the number of iterations $T$ is 10.
We run each setup with 10 random seeds and the search cost for a run is 8 GPU hours on GeForce GTX 1080 Ti.
\begin{table*}[t]
   \centering
   \small
	\begin{tabular}{|l|c|c|c|} \hline
          \textbf{Algorithm}    &  \textbf{Mean Acc ($\%$)} & \textbf{Mean Rank} &\textbf{Average \#Q}   \\ \hline
		 Random Search    & 93.84 $\pm$ 0.13  & 498.80 & 648     \\  
		 Random Search    & 93.92 $\pm$ 0.11  & 211.50 & 1562      \\  		
		\textbf{GA-NAS}-Setup1    &\textbf{94.22} $\pm$ 4.45e-5 & 2.90 & 647.50 $\pm$ 433.43    \\        
        \textbf{GA-NAS}-Setup2    &\textbf{94.23} $\pm$ 7.43e-5 & 2.50 & 1561.80 $\pm$ 802.13  \\
		\hline
	\end{tabular}
	%\vspace{-1mm}
	\caption{The average statistics of the best cells found on NAS-Bench-101 without weight sharing, averaged over 10 runs (with std shown). Note that we set the number of queries (\#Q) for Random Search to be the same as the average number of queries incurred by GA-NAS.}% to find the best architecture.} 
	\label{table:nb101_avg_acc}
	%\vspace{-2mm}
\end{table*}

Table~\ref{table:nb101_best_acc} compares GA-NAS to other methods for the best cell that can be found by querying NAS-Bench-101, in terms of the accuracy and the rank of this cell in NAS-Bench-101, along with the number of queries required to find that cell. Table~\ref{table:nb101_avg_acc} shows the average performance of GA-NAS in the same experiment over multiple random seeds. Note that Table~\ref{table:nb101_best_acc} does not list the average performance of other methods except Random Search, since all the other methods in Table~\ref{table:nb101_best_acc} only reported their single-run performance on NAS-Bench-101 in their respective experiments.

In both tables, we find that GA-NAS can reach a higher accuracy in fewer number of queries, and beats the best published results, i.e., BANANAS~\cite{white2019bananas} and SemiNAS~\cite{luo2020semi} by an obvious margin. Note that 94.22 is the 3rd best cell while 94.23 is the 2nd best cell in NAS-Bench-101. From Table~\ref{table:nb101_avg_acc}, we observe that GA-NAS achieves superior stability and reproducibility: GA-NAS-setup1 consistently finds the 3rd best in 9 runs and the 2nd best in 1 run out of 10 runs; GA-NAS-setup2 finds the 2nd best in 5 runs and the 3rd best in the other 5 runs. 

To evaluate GA-NAS when true accuracy is not available, we train a weight-sharing supernet on the search space of NAS-Bench-101 (with details provided in Appendix) and report the true test accuracies of architectures found by GA-NAS. We use the supernet to evaluate the accuracy of a cell on a validation set of 10k instances of CIFAR10 (see Appendix). Search time including supernet training is around 2 GPU days. 

\begin{table}[t]{}
	% \vspace{-2mm}
	\centering
  \small
	\begin{tabular}{|l|c|c|c|} 
	\hline
         \textbf{Algorithm}              & \textbf{Mean Acc}         & \textbf{Best Acc} &  \textbf{Best Rank}  \\ \hline
		DARTS $^\dagger$          & 92.21 $\pm$ 0.61          & 93.02             & 57079                 \\
		NAO $^\dagger$            & 92.59 $\pm$ 0.59          & 93.33             & 19552                  \\
		ENAS $^\dagger$           & 91.83 $\pm$ 0.42          & 92.54             & 96939                   \\
        \hline
		\textbf{GA-NAS}         & \textbf{92.80 $\pm$ 0.54} & \textbf{93.46}    & \textbf{5386}          \\
		\hline
	\end{tabular}
	%\vspace{-2mm}
	\caption{\label{table:nb101_supernet_acc_search}  Searching on NAS-Bench-101 with weight-sharing, with the mean test accuracy of the best cells from 10 runs, and the best accuracy/rank found by a single run. $\dagger$: taken from \protect\cite{yu2019evaluating}}  
	%\vspace{-2mm}
\end{table}

We report results of 10 runs in Table~\ref{table:nb101_supernet_acc_search}, in comparison to other weight-sharing NAS schemes reported in \cite{yu2019evaluating}. We observe that using a supernet degrades the search performance in general as compared to true evaluation, because weight-sharing often cannot provide a completely reliable performance for the candidate architectures. Nevertheless, GA-NAS outperforms other approaches.
\begin{table*}[t]
	\centering
	\small
	%\scalebox{0.925}{
	\begin{threeparttable}[hb]
	\begin{tabular}{|l||c|c|c||c|c|c||c|c|c|}
		\hline
        & \multicolumn{3}{c||}{\textbf{CIFAR-10}} & \multicolumn{3}{c||}{\textbf{CIFAR-100}}  & \multicolumn{3}{c|}{\textbf{ImageNet-16-120}}\\ \hline
		\textbf{Algorithm}           & \textbf{Mean Acc}  & \textbf{Rank}&  \textbf{\#Q} & \textbf{Mean Acc}  & \textbf{Rank}&  \textbf{\#Q} & \textbf{Mean Acc}  & \textbf{Rank}&  \textbf{\#Q} \\ \hline
		REA \tnote{$\dagger$}        & 93.92 $\pm$ 0.30          & -             & -      & 71.84 $\pm$ 0.99          & -             & -      & 45.54 $\pm$ 1.03          & -     & - \\
		RS \tnote{$\dagger$}         & 93.70 $\pm$ 0.36          & -             & -      & 71.04 $\pm$ 1.07          & -             & -      & 44.57 $\pm$ 1.25          & -     & - \\
		REINFORCE \tnote{$\dagger$}  & 93.85 $\pm$ 0.37          & -             & -      & 71.71 $\pm$ 1.09          & -             & -      & 45.24 $\pm$ 1.18          & -     & - \\
		BOHB \tnote{$\dagger$}       & 93.61 $\pm$ 0.52          & -             & -      & 70.85 $\pm$ 1.28          & -             & -      & 44.42 $\pm$ 1.49          & -     & - \\
		RS-500                       & 94.11 $\pm$ 0.16          & 30.81         & 500    & 72.54 $\pm$ 0.54          & 30.89         & 500    & 46.34 $\pm$ 0.41          & 34.18 & 500 \\
		\hline
		\textbf{GA-NAS}              & \textbf{94.34 $\pm$ 0.05} & \textbf{4.05} & 444 & \textbf{73.28 $\pm$ 0.17} & \textbf{3.25} & 444 & \textbf{46.80 $\pm$ 0.29} & \textbf{7.40} & 445 \\
		\hline
	\end{tabular}
	\begin{tablenotes}\footnotesize
	    \item{$\dagger$} The results are taken directly from NAS-Bench-201~\cite{dong2020nasbench201}.
	\end{tablenotes}
	%\vspace{-2mm}
	\caption{\label{table:nb201}Searching on NAS-Bench-201 without weight sharing, with the mean accuracy and rank of the best cell found reported. \#Q represents the average number of queries per run. 
	We conduct 20 runs for GA-NAS.}
	\end{threeparttable}
	%}
	%\vspace{-2mm}
\end{table*}

%\textbf{NAS-Bench-201} 
%\paragraph{NAS-Bench-201.} % This needs to be a subsubsection
\subsubsection{NAS-Bench-201}
This search space contains 15,625 evaluated cells. 
Architecture cells %The search space 
consist of 6 searchable edges and 5 candidate operations.
We test GA-NAS on NAS-Bench-201 by conducting 20 runs for CIFAR-10, CIFAR-100, and ImageNet-16-120 using the true test accuracy. We compare against the baselines from the original NAS-Bench-201 paper \cite{dong2020nasbench201} \textit{that are also directly querying the benchmark data.} Since no information on the rank achieved and the number of queries is reported for these baselines, we also compare GA-NAS to Random Search (RS-500), which evaluates 500 unique cells in each run. Table~\ref{table:nb201} presents the results. 
We observe that GA-NAS outperforms a wide range of baselines, including Evolutionary Algorithm (REA), Reinforcement Learning (REINFORCE), and Bayesian Optimization (BOHB), on the task of finding the most accurate cell with much lower variance.
Notably, on ImageNet-16-120, GA-NAS outperforms REA by nearly $1.3$\% on average. 

Compared to RS-500, GA-NAS finds cells that are higher ranked while only exploring less than 3.2\% of the search space in each run. 
It is also worth mentioning that in the $20$ runs on all three datasets, GA-NAS can find the best cell in the entire benchmark more than once. 
Specifically for CIFAR-10, it found the best cell in 9 out of 20 runs.
%%%\footnote{There are two cells in NAS-Bench-201 with the same, highest CIFAR-10 test accuracy of 94.37. We record that GA-NAS finds the best cell if either one is found.}.

%\textbf{NAS-Bench-301}
%\paragraph{NAS-Bench-301.} % Need to make subsubsection
\subsubsection{NAS-Bench-301}
This is another recently proposed benchmark based on the same search space as DARTS. 
Relying on surrogate performance models, NAS-Bench-301 reports the accuracy of $10^{18}$ unique cells. 
We are especially interested in how the number of queries (\#Q) needed to find an architecture with high accuracy scales in a large search space.
We run GA-NAS on NAS-Bench-301 v0.9. 
We compare with Random (RS) and Evolutionary (EA) search baselines. 
Figure~\ref{fig:nb301_comp} plots the average best accuracy along with the accuracy standard deviations versus the number of queries incurred under the three methods.
We observe that GA-NAS outperforms RS at all query budgets and outperforms EA when the number of queries exceeds 3k.
The results on NAS-Bench-301 confirm that for GA-NAS, the number of queries (\#Q) required to find a good performing cell scales well as the size of the search space increases. 
For example, on NAS-Bench-101, GA-NAS usually needs around 500 queries to find the 3rd best cell, with an accuracy $\approx94\%$ among 423k candidates, while on the huge search space of NAS-Bench-301 with up to $10^{18}$ candidates, it only needs around 6k queries to find an architecture with accuracy approximately equal to $95\%$.

In contrast to GA-NAS, EA search is less stable and does not improve much as the number of queries increases over 3000. GA-NAS surpasses EA for \#Q $\geq 3000$ in Figure~\ref{fig:nb301_comp}. What is more important is that the variance of GA-NAS is much lower than the variance of the EA solution over all ranges of \#Q. 
Even though for \#Q < 3000, GA-NAS and EA are close in terms of average performance, EA suffers from a huge variance.

\begin{figure}[t]
  %\vspace{-2mm}
  \begin{center}
    \includegraphics[width=0.45\textwidth]{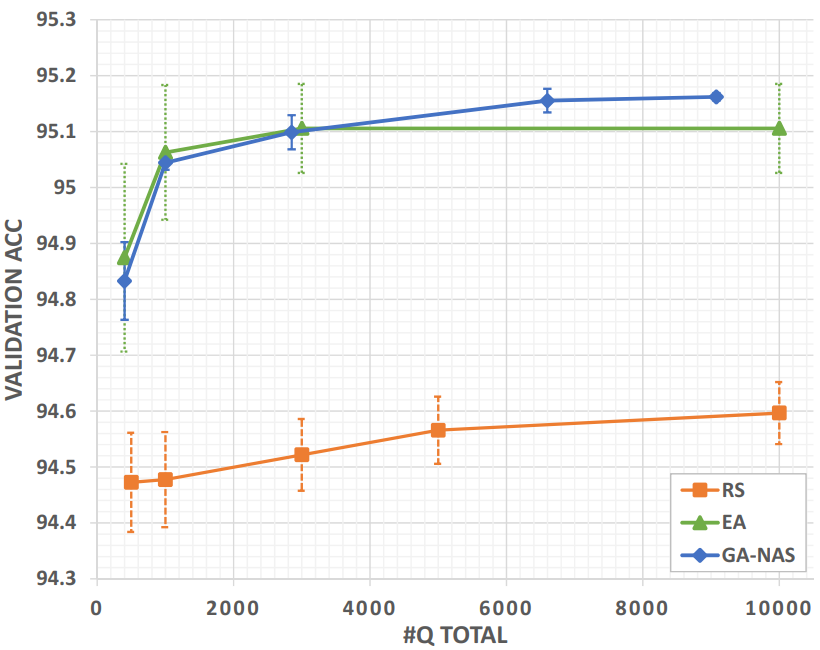}
  \end{center}
  %\vspace{-3mm}
  \caption{\label{fig:nb301_comp}NAS-Bench-301 results comparing the means/standard deviations of the best accuracy found at various total query limits.}
  %\vspace{-4mm}
\end{figure}

%{\bf Ablation Studies.}
%\paragraph{Ablation Studies.} % Must be subsubsection
\subsubsection{Ablation Studies}
A key question one might raise about GA-NAS is how much the discriminator contributes to the superior search performance.
Therefore, we perform an ablation study on NAS-Bench-101 by creating an \textit{RL-NAS} algorithm for comparison.
RL-NAS removes the discriminator in GA-NAS and directly queries the accuracy of a generated cell from NAS-Bench-101 as the reward for training.
We test the performance of RL-NAS under two setups that differ in the total number of queries made to NAS-Bench-101. 
Table~\ref{table:nb101_ablation_1} reports the results.

\begin{table}[t]
   \centering
   \small
   \scalebox{0.96}{
	\begin{tabular}{|l|c|c|c|} 
	\hline
    \textbf{Algorithm} & \textbf{Avg. Acc} & \textbf{Avg. Rank} & \textbf{Avg. \#Q} \\
    \hline
     RL-NAS-1 & 94.14 $\pm$ 0.10  & 20.8 & 7093 $\pm$ 3904\\
     RL-NAS-2 & 93.78 $\pm$ 0.14 & 919.0  & 314 $\pm$ 300 \\
	 GA-NAS-Setup1 & 94.22$\pm$ 4.5e-5 & 2.9 & 648 $\pm$ 433 \\
	 GA-NAS-Setup2 & 94.23$\pm$ 7.4e-5 & 2.5 & 1562 $\pm$ 802 \\
	\hline
	\end{tabular}
	%\vspace{-3mm}
	}
	\caption{ \label{table:nb101_ablation_1} Results of ablation study on NAS-Bench-101 by removing the discriminator and directly queries the benchmark for reward.}  
	%\vspace{-2mm}
\end{table}

Compared to GA-NAS-Setup2, RL-NAS-1 makes 3$\times$ more queries to NAS-Bench-101, yet cannot outperform either of the GA-NAS setups.
If we instead limits the number of queries as in RL-NAS-2 then the search performance deteriorates significantly. 
Therefore, we conclude that the discriminator in GA-NAS is crucial for reducing the number of queries, which converts to the number of evaluations in real-world problems, as well as for finding architectures with better performance. 
Interested readers are referred to Appendix for more ablation studies as well as the Pareto Front search results.

\subsection{Improving Existing Neural Architectures}
\begin{table}[b]
	\centering
	\small
	\scalebox{0.97}{
	\begin{threeparttable}[hb]
	\begin{tabular}{|l|c|c|c|}
		\hline 
        \multicolumn{4}{|c|}{\textbf{ResNet}} \\ \hline
        \textbf{Algorithm}
        & \textbf{Best Acc}  & \textbf{Train Seconds}&  \textbf{\#Weights (M)} \\ \hline
		Hand-crafted    & 93.18          & 2251.6          & 20.35 \\
		Random Search    & 93.84          & 1836.0          & 10.62 \\
		\textbf{GA-NAS} & \textbf{93.96} & 1993.6 & 11.06  \\ \hline
		 \multicolumn{4}{|c|}{\textbf{Inception}} \\ \hline
		 \textbf{Algorithm} & \textbf{Best Acc}  & \textbf{Train Seconds}&  \textbf{\#Weights (M)} \\ \hline
		 Hand-crafted    & 93.09          & 1156.0          & 2.69  \\ 
		 Random Search    & 93.14          & 1080.4          & 2.18 \\
		\textbf{GA-NAS} & \textbf{93.28} & 1085.0 & 2.69  \\
		\hline
	\end{tabular}
	%\vspace{-2mm}
	\end{threeparttable}
	}
	\caption{Constrained search results on NAS-Bench-101. GA-NAS can find cells that are superior to the %hand-crafted 
	ResNet and Inception cells in terms of test accuracy, training time, and the number of weights.}% (\#Weights, in millions).}
	\label{table:nb101_cons_acc_search}
	%\vspace{-2mm}
\end{table}

We now demonstrate that GA-NAS can improve existing neural architectures, including ResNet %\cite{he2016deep} 
and Inception %\cite{szegedy2016rethinking} 
cells in NAS-Bench-101, EfficientNet-B0 under hard constraints, and ProxylessNAS-GPU \cite{cai2018proxylessnas} in unconstrained search.

For ResNet and Inception cells, we use GA-NAS to find better cells from NAS-Bench-101 under a lower or equal training time and number of weights. This can be achieved by enforcing a hard constraint in choosing the truth set $\mathcal T$ in each iteration.
%such that when updating the current truth set $\mathcal T$, we ensure that each cell in this set has a lower or equal training time and number of weights.
Table~\ref{table:nb101_cons_acc_search} shows that GA-NAS can find new, dominating cells for both cells, showing that it can enforce ad-hoc constraints in search, a property not enforceable by regularizers in prior work.
We also test Random Search under a similar number of queries to the benchmark under the same constraints, which is unable to outperform GA-NAS.

% \red{In the context of macro search, GA-NAS searches for a single-path network that consists of a sequence of blocks, where each block is chosen from a pool of predefined candidates.
% In contrast to the micro search agent, we replace the GNN in our encoder with a GRU to make the model more lightweight. We also constraint the actor to outputting probability distribution over the operations and drop the edge connection prediction, which is only necessary for the micro search.}

% We now use GA-NAS to improve EfficientNet-B0 network \cite{tan2019efficientnet}.
% In the context of macro search, GA-NAS searches for a single-path network that consists of a sequence of blocks, where each block is chosen from a set of predefined candidates.

We now consider well-known architectures found on ImageNet %\cite{russakovsky2015imagenet}, 
i.e., EfficientNet-B0 and ProxylessNAS-GPU, which are already optimized strong baselines found by other NAS methods.
We show that GA-NAS can be used to improve a given architecture in practice by searching in its original search space.

For EfficientNet-B0, we set the constraint that the found networks all have an equal or lower number of parameters than EfficientNet-B0.
For the ProxylessNAS-GPU model, we simply put it in the starting truth set and run an unconstrained search to further improve its top-1 validation accuracy. More details are provided in the Appendix. 
Table~\ref{table:eb0_search_results} presents the improvements made by GA-NAS over both existing models.
Compared to EfficientNet-B0, GA-NAS can find new single-path networks that achieve comparable or better top-1 accuracy on ImageNet with an equal or lower number of trainable weights. 
We report the accuracy of EfficientNet-B0 and the GA-NAS variants \textit{without data augmentation}.
% Total search time including supernet training is around 21 GPU days on Tesla V100 GPUs (20 GPU days for supernet training and 1 GPU day for the search). \red{
Total search time including supernet training is around 680 GPU hours on Tesla V100 GPUs.
It is worth noting that the original EfficientNet-B0 is found using the MNasNet \cite{tan2018mnasnet} with a search cost over 40,000 GPU hours.

\begin{table}[t]{}
   \centering
   \small
	\begin{tabular}{|l|c|c|c|} 
	\hline
    \textbf{Network} & \textbf{\#Params} & \textbf{Top-1 Acc}\\
    \hline
     %ResNet-50 \cite{he2016deep} & $26$M & $76.0$ \\
     %DenseNet-169 \cite{huang2017densely} & $14$M & $76.2$ \\
	 EfficientNet-B0 (no augment) & $5.3$M & $76.7$ \\
	 GA-NAS-ENet-1 & $\textbf{4.6}$M & $76.5$ \\
	 GA-NAS-ENet-2 & $\textbf{5.2}$M & $\textbf{76.8}$ \\
	 GA-NAS-ENet-3 & $5.3$M & $\textbf{76.9}$ \\
	 \hline
	 ProxylessNAS-GPU & $4.4$M & $75.1$ \\
	 GA-NAS-ProxylessNAS & $4.9$M & $\textbf{75.5}$ \\
	\hline
	\end{tabular}
	%\vspace{-2mm}
	\caption{ \label{table:eb0_search_results} Results on the EfficientNet and ProxylessNAS spaces.}  
	%\vspace{-2mm}
\end{table}
For ProxylessNAS experiments, we train a supernet on ImageNet %for around 20 GPU days, 
and conduct an unconstrained search using GA-NAS for 38 hours on 8 Tesla V100 GPUs, a major portion of which, i.e., 29 hours is spent on querying the supernet for architecture performance. 
Compared to ProxylessNAS-GPU, GA-NAS can find an architecture with a comparable number of parameters and a better top-1 accuracy.
% we constrain the number of trainable weights to be no more than the size of EfficientNet-B0, while trying to further increase the validation accuracy on ImageNet. 
%Table~\ref{table:eb0_search_results} summarizes the results. 

% We also visualize the structures of the found networks in the Appendix.

\section{Conclusion}\label{sec:conclusion}
In this paper, we propose Generative Adversarial NAS (GA-NAS), a provably converging search strategy for NAS problems, based on generative adversarial learning and the importance sampling framework. 
Based on extensive search experiments performed on NAS-Bench-101, 201, and 301 benchmarks, we demonstrate the superiority of GA-NAS as compared to a wide range of existing NAS methods, in terms of the best architectures discovered, convergence speed, scalability to a large search space, and more importantly, the stability and insensitivity to random seeds.
We also show the capability of GA-NAS to improve a given practical architecture, including already optimized ones either manually designed or found by other NAS methods, and its ability to search under ad-hoc constraints. GA-NAS improves EfficientNet-B0 by generating architectures with higher accuracy or lower numbers of parameters, and improves ProxylessNAS-GPU with enhanced accuracies. 
These results demonstrate the competence of GA-NAS as a stable and reproducible search algorithm for NAS.
%These results indicate that GA-NAS generalizes well to diverse types of search spaces and can improve already optimized and strongly-performing neural architectures in their respective search spaces for large-scale image classification tasks.

\bibliographystyle{named}
\bibliography{ijcai21}

\clearpage
\appendix
\section{Appendix}\label{sec:sup}
\subsection{Importance Sampling and the Cross-Entropy Method}
\label{sec:crossEntropy}
Assume $X$ is a random variable, taking values in $\mathcal X$ and has a prior probability density function (pdf) $p(.;\bar\lambda)$ for fixed $\bar\lambda\in \Theta$.
Let $S(X)$ be the objective function to be maximized, and $\alpha$ be a level parameter. 
Note that the event $\mathcal E:=\left\{S(X)\geq \alpha\right\}$ is a \textit{rare event} for an $\alpha$ that is equal to or close to the optimal value of $S(X)$.
The goal of \emph{rare-event} probability estimation is to estimate 
\begin{eqnarray}
    l(\alpha) &:=& \mathbb{P}_{\bar\lambda}\left(S(X) \geq \alpha\right)\nonumber\\ &=&\mathbb{E}_{\bar\lambda}\left[\mathbb{I}_{S(X)\geq\alpha}\right] = \int \mathbb{I}_{S(x)\geq\alpha} ~p(x;\bar\lambda)dx,\nonumber
\end{eqnarray}
where $\mathbb{I}_{x\in \mathcal E}$ is an indicator function that is equal to 1 if $x\in \mathcal E$ and 0 otherwise. 

In fact, $l(\alpha)$ is called the \emph{rare-event} probability (or expectation) which is very small, e.g., less than $10^{-4}$.
The general idea of \emph{importance sampling} is to estimate the above rare-event probability $l$ by drawing samples $x$ from important regions of the search space with both a large density $p(x;\bar\lambda)$ and a large $\mathbb{I}_{S(x)\geq\alpha}$, i.e., $\mathbb{I}_{S(x)\geq\alpha}=1$. In other words, we aim to find $x$ such that $S(x)\geq\alpha$ with high probability. 

Importance sampling estimates $l$ %a property related to the prior density $p(x;\bar\lambda)$,
by sampling $x$ from a distribution $q^*(x,\alpha;\bar\lambda)$ %$p(x;\theta)$ 
that should be proportional to
$\mathbb{I}_{S(x)\geq\alpha}p(x;\bar\lambda)$.
%(see \cite{murphy2012machine} page 820).
 Specifically, define the proposal sampling density $q(x)$ as a function such that $q(x) = 0$ implies $\mathbb{I}_{S(x)\geq\alpha} p(x;\bar\lambda) = 0$ for every $x$. 
Then we have
\begin{equation}\label{eq:importance1}
    l(\alpha)  = \int \frac{\mathbb{I}_{S(x)\geq\alpha} p(x;\bar\lambda)}{q(x)}q(x)dx = \mathbb{E}_q\left[\frac{\mathbb{I}_{S(X)\geq\alpha} p(X;\bar\lambda)}{q(X)}\right].
\end{equation}
\cite{rubinstein2013cross} shows that the optimal importance sampling probability density $q$ which minimizes the variance of the empirical estimator of $l(\alpha)$ is the density of $X$ conditional on the event $S(X) \geq \alpha$, that is 
\begin{equation}\label{eq:CE_Opt1}
    q^*(x,\alpha;\bar\lambda) = \frac{p(x;\bar\lambda)\mathbb{I}_{S(x)\geq\alpha}}{\int \mathbb{I}_{S(x)\geq\alpha} ~p(x;\bar\lambda)dx}.
\end{equation}

However, noting that the denominator of (\ref{eq:CE_Opt1}) is the quantity $l(\alpha)$ that we aim to estimate in the first place, we cannot obtain an explicit form of $q^*(.)$ directly. 
To overcome this issue, the CE method \cite{homem2002rare} aims to choose the sampling probability density $q(.)$ from the parametric families of densities $\{p(.;\theta)\}$ such that the Kullback-Leibler (KL)-divergence between the optimal importance sampling probability density $q^*(.,\alpha;\bar\lambda)$, given by (\ref{eq:CE_Opt1}), and $p(.;\theta)$ is minimized. 

%proposes the Cross-Entropy (CE) method to estimate $q^*(.)$ by generating a sequence of probability densities $p(.;\theta_1)$, $p(.;\theta_2)$, \ldots, that approaches $q^*(.)$ in terms of KL divergence \cite{homem2002rare}.

%Furthermore, motivated by the well-known fact that the symmetric Jensen-Shannon (JS) divergence is more robust than the asymmetric KL divergence \cite{nowozin2016f}, we replace the KL divergence with the JS divergence in our application of importance sampling.

%\subsection{Convergence Results}
%\label{sec:convergence}

%The KL-divergence between $q^*(.)$ and $p(.;\theta)$ is given by
%\begin{equation}
%    KL(q^*||p) = \int q^*(x)\ln\frac{q^*(x)}{p(x;\theta)}dx = \mathbb{E}_q^*\left[\ln\frac{q^*(X)}{p(X;\theta)}\right].
%\end{equation}
%Noting the assumption that $X$ is also distributed according to the parametric family $\{p(.;\theta)\}$ parameterized by $\bar\lambda$, and assuming $\Tilde{\theta}$ is some prior parameter for the sampling distribution, from (\ref{eq:CE_Opt1}), 
The \emph{CE-Optimal solution} $\theta^*$ is given by
\begin{eqnarray}\label{eq:opt}
\theta^* &=& \arg\min_{\theta} KL\left(q^*(.,\alpha;\bar\lambda)||p(.;\theta)\right) \nonumber\\
&=& \arg\max_\theta\mathbb{E}_{\bar\lambda} \left[\mathbb{I}_{S(X)\geq \alpha}\ln p(X;\theta)\right]\nonumber\\
&=& \arg\max_\theta \mathbb{E}_{\Tilde{\theta}} \left[\mathbb{I}_{S(X)\geq \alpha}\ln p(X;\theta) \frac{p(X;\bar\lambda)} {p(X;\Tilde{\theta})}\right],
\end{eqnarray}
given any prior sampling distribution parameter $\Tilde{\theta}\in\Theta$.
%\begin{eqnarray}\label{eq:opt}
%\theta^*&=&\argmin_{\theta} KL\left(q^*||p(.;\theta)\right)
%= \argmax_\theta{\mathbb{E}_\Tilde{\theta}} \left[\mathbb{I}_{S(X)\geq \alpha}\ln p(X;\theta)\frac{p(X;\bar\lambda)}{p(X;\Tilde{\theta})}\right].
%\end{eqnarray}
%\begin{eqnarray}\label{eq:opt}
%\theta^*&=&\argmin_{\theta} KL\left(q^*||p(.;\theta)\right)
%=\argmax_\theta\mathbb{E}_\bar\lambda \left[\mathbb{I}_{S(X)\geq \alpha}\ln p(X;\theta)\right]\nonumber\\
%&=&\argmax_\theta{\mathbb{E}_\Tilde{\theta}} \left[\mathbb{I}_{S(X)\geq \alpha}\ln p(X;\theta)\frac{p(X;\bar\lambda)}{p(X;\Tilde{\theta})}\right].
%\end{eqnarray}
Let $x_1,\ldots,x_N$ be  i.i.d samples from $p(.;\Tilde{\theta})$. Then an empirical estimate of $\theta^*$ is given by
\begin{equation}\label{eq:emp_opt}
    \hat{\theta^*} = \arg\max_\theta \frac{1}{N}\sum_{k = 1}^N\mathbb{I}_{S(x_k)\geq \alpha}\frac{p(x_k;\bar\lambda)}{p(x_k;\Tilde{\theta})} \ln p(x_k;{\theta}).
\end{equation}

In the case of rare events, where $l(\alpha) \leq 10^{-6}$, solving the optimization problem (\ref{eq:emp_opt}) does not help us estimate the probability of the rare event (see \cite{homem2002rare}). Algorithm~\ref{alg:alg1} presents the \emph{multi-stage} CE approach proposed in \cite{homem2002rare} to overcome this problem.
This iterative CE algorithm essentially creates a sequence of sampling probability densities $ p(.;\theta_1), p(,;\theta_2),\ldots$ that are steered toward the direction of the theoretically optimal density $q^*(.,\alpha;\bar\lambda)$ in an iterative manner.
%In this multi-stage iterative algorithm, an auxiliary threshold sequence ${\gamma_t}$, $t\geq0$, is introduced and the algorithm iterates between updating $\gamma_t$ and $\theta_t$. 

\begin{algorithm}[h]
	\caption{CE Method for rare-event estimation} \label{alg:alg1}
	\begin{algorithmic}[1]
		\State\textbf{Input:} Level parameter $\alpha > 0 $, some fixed parameter $\delta > 0$, and initial 
	    parameters $\bar\lambda$ and $\rho$.
	    \State $t\leftarrow 1$, $\rho_0\leftarrow\rho$, $\theta_0\leftarrow\bar\lambda$
	    \While {$\gamma(\theta_{t-1}, \rho_{t-1}) < \alpha$}
	        \State $\gamma_{t-1} \leftarrow  \min (\alpha, \gamma(\theta_{t-1}, \rho_{t-1}))$
	        \State \label{step:kl} \begin{equation}
	        \begin{multlined}[c]
	        \theta_t \in \arg\max_{\theta\in\Theta} \mathbb{E}_{\bar\lambda} \left[\mathbb{I}_{S(X) \geq \gamma_{t-1}}\ln p(X;\theta)\right]\nonumber\\
	        = \arg\min_{\theta\in\Theta} KL(q^*(.,\gamma_{t-1};\bar\lambda) || p(.;\theta))
	        \end{multlined}
	        \end{equation}
	        \State Set $\rho_t$ such that $\gamma(\theta_t, \rho_t) \geq \min(\alpha, \gamma(\theta_{t-1}, \rho_{t-1}) + \delta)$.
	        \State $t\leftarrow t+1$.
	        
	       % $\theta_t \in \arg\max_{\theta\in\Theta} \mathbb{E}_{\bar\lambda} \left[\mathbb{I}_{S(X) \geq \gamma_{t-1}}\ln p(X;\theta)\right]
	       % =
	       % \arg\min_{\theta\in\Theta} KL(q^*(.,\gamma_{t-1};\bar\lambda) || p(.;\theta))$
	       % \State Set $\rho_t$ such that $\gamma(\theta_t, \rho_t) \geq \min(\alpha, \gamma(\theta_{t-1}, \rho_{t-1}) + \delta)$ for some fix $\delta > 0$.
	       % \State $t\leftarrow t+1$.
	    \EndWhile
	    %\EndWhile{$X $} 
	\end{algorithmic}
\end{algorithm}

More precisely, in Algorithm~\ref{alg:alg1}, we denote the threshold sequence by ${\gamma_t}$ for $t\geq 0$, and sampling distribution parameters by $\theta_t$ for $t\geq 0$. 
Initially, choose $\rho$ and $\gamma(\bar \lambda,\rho)$ so that $\gamma(\bar \lambda,\rho)$ is the $(1-\rho)$-quantile of $S(X)$ under the density $p(.;\bar\lambda)$, and generally, let $\gamma(\theta_t, \rho_t)$ be the $(1-\rho_t)$-quantile of $S(X)$ under the sampling density $p(x;\theta_t)$ of iteration $t$.
 %initially, choose $\rho$ and $\gamma(\bar \lambda,\rho)$ so that $\gamma(\bar \lambda,\rho)$ is the $(1-\rho)$-quantile of $S(X)$ under density $p(.;\bar\lambda)$. And, in general, let $\gamma(\theta_t, \rho_t)$ be the $(1-\rho_t)$-quantile of $S(X)$ under the probability density $p(X;\theta_t)$. 
Furthermore, the \textit{a priori} sampling density parameter $\Tilde{\theta}$ introduced in (\ref{eq:opt}) and (\ref{eq:emp_opt}) is replaced by the optimum parameter obtained in iteration $t-1$, i.e., $\theta_{t-1}$.

%in Algorithm~\ref{alg:alg1}, $\gamma(\theta_t, \rho_t)$ is the $(1-\rho_t)$-quantile of $S(X)$ under the probability density $p(X;\theta_t)$. Furthermore, the prior parameter $\Tilde\theta$ introduced in (\ref{eq:opt}) and (\ref{eq:emp_opt}) \red{is replaced by} the optimum parameter obtained in iteration $t-1$, i.e., $\theta_{t-1}$.

% \begin{algorithm}
% 	\caption{CE} 
% 	\begin{algorithmic}[1]
% 		\For {$iteration=1,2,\ldots$}
% 			\For {$actor=1,2,\ldots,N$}
% 				\State Run policy $\pi_{\theta_{old}}$ in environment for $T$ time steps
% 				\State Compute advantage estimates $\hat{A}_{1},\ldots,\hat{A}_{T}$
% 			\EndFor
% 			\State Optimize surrogate $L$ wrt. $\theta$, with $K$ epochs and minibatch size $M\leq NT$
% 			\State $\theta_{old}\leftarrow\theta$
% 		\EndFor
% 	\end{algorithmic} 
% \end{algorithm}

From (\ref{eq:opt}), we notice that Step 4 in Algorithm~\ref{alg:alg1} is equivalent to minimizing the KL-divergence between the following two densities, $q(x, \gamma_{t-1};\bar\lambda) = c^{-1}\mathbb{I}_{S(x)\geq\gamma_{t-1}}p(x;\bar\lambda)$ and $p(x;\theta)$, where $c = \int \mathbb{I}_{S(x)\geq\gamma_{t-1}}p(x;\bar\lambda)dx$, $\gamma_{t-1} = \min (\alpha, \gamma(\theta_{t-1}, \rho_{t-1}))$.
One can always choose a small $\delta$ in Step 5, which will determine the number of times the while loop is executed.
% We will show later that one can always choose a small $\delta$ in Step 5, which will determine the number of times the while loop is executed.

% {Suppose that Algorithm~\ref{alg:alg1} terminates with an estimate $\theta_T$ of $\theta^*$.
% Letting $x_1,\ldots,x_N$ be independent and identically distributed (i.i.d.) samples from the probability density $p(.,\theta_T)$, an \emph{unbiased empirical estimator} of $l$ can be obtained as 
% \begin{equation}
%     \hat{l}(\alpha)= \frac{1}{N}\sum_{k=1}^N\mathbb{I}_{S(x_k)\geq\alpha}\frac{p(x_k;\bar\lambda)}{p(x_k,\theta_k)}.
% \end{equation}
% }

% In order to use Algorithm~\ref{alg:alg1} in an optimization framework where the goal is to maximize $S(X)$ and there is no specific level parameter $\alpha$, one can use any desired stopping criteria to terminate the while loop in Algorithm~\ref{alg:alg1} (see \cite{botev2013cross}). For example, one can stop the algorithm when no change in the value of $S(X)$ is observed after a certain number of iterations (see \cite{rubinstein2013cross} page 134).
%In order to maximize $S(x)$, we only need to make the following changes to Algorithm~\ref{alg:alg1}... Let $\gamma^*$ be the optimal value of $S(X)$.} 

%\subsection{Convergence Results and Other Probability Distance Measures}
%\label{sec:convergence}
%We say $\theta^*$ is a \emph{CE-optimal solution}, if 
%\begin{align}\label{eq:CE_Opt}
%    \theta^* \in \argmax_{\theta} \mathbb{E}_{\bar\lambda} \left[I_{S(X) \geq \alpha} \ln{p\left(X;\theta\right)}\right].
%\end{align}
The following theorem which was first proved by \cite{homem2002rare} asserts that Algorithm~\ref{alg:alg1} terminates. We provide a proof of Theorem~\ref{thm:Thm1} as follows. %Section~\ref{sec:sup_proofs}.
\begin{theorem}\label{thm:Thm1}
%Let $\Theta^*$ be the set of maximizers of (\ref{eq:opt}). 
%Let $\bar\Theta = \{\theta\in\Theta: P\left(S(X)\geq \alpha;\theta\right) > 0\}$.
%Assume that there exist $\theta^*\in \Theta^*$ with 
%$P\left(S(X)\geq \alpha;\theta^*\right)>0$.
%
%$\Theta^*\cap \bar\Theta \neq \emptyset$.
%We assume that
%there exists a set $\bar\Theta$ such that $\Theta^*\cap \bar\Theta \neq %\emptyset$, and  $P\left(S(X)\geq \alpha;\theta\right) > 0$ for all $\theta\in \bar\Theta$. 
Assume that $l(\alpha)>0$.
Then, Algorithm~\ref{alg:alg1} converges to a CE-optimal solution in a finite number of iterations.
\end{theorem}
%\begin{remark}
%The assumption in the Theorem~\ref{thm:Thm1} implies that the probability $P(S(X)\geq \alpha;\theta^*)$ does not vanish in the neighbourhood of the optimal parameter $\theta^*$. 
%The assumption holds, for instance, when $l>0$, or when for each $\theta\in\Theta$, $p(x,\theta)>0$ for all $x \in \mathcal{X}$.
%\end{remark}
\begin{proof}
Let $t$ be an arbitrary iteration of the algorithm, and let $\rho_\alpha := P(S(X)\geq \alpha;\theta_t)$. 
%By the assumption of the theorem,
%it can be shown by induction that $\theta_t\in\bar\Theta$, thus, we have %$\rho_\alpha > 0$.
It can be shown by induction that  $\rho_\alpha > 0$. Note that this is also true if $p(x;\theta)>0$ for all \(\theta,x\).
Next, we show that for every $\rho_t\in(0, \rho_\alpha)$, $\gamma(\theta_t, \rho_t) \geq \alpha$. By the definition of $\gamma$, we have 
\begin{eqnarray}\label{eq:inequalities1}
&&P(S(X) \geq \gamma(\theta_t, \rho_t);\theta) \geq \rho_t,\nonumber\\
&&P(S(X) \leq \gamma(\theta_t, \rho_t);\theta) \geq 1 - \rho_t > 1 - \rho_\alpha.
\end{eqnarray}
Suppose by contradiction, that $\gamma(\theta_t, \rho_t) < \alpha$. Then, we get
\begin{align*}
    P(S(X) \leq \gamma(\theta_t, \rho_t);\theta^*) \leq P(S(X) \leq \alpha; \theta_t) = 1 - \rho_\alpha,
\end{align*}
which contradicts (\ref{eq:inequalities1}). Thus, $\gamma(\theta_t, \rho_t) \geq \alpha$. This implies that Step 5 can always be carried out for any $\delta>0$.
%and the while loop in Algorithm~\ref{alg:alg1} terminates. 
Consequently, Algorithm~\ref{alg:alg1} terminates after $T$ iterations with $T\leq \ceil{\alpha / \delta}$. Thus at time $T$, we have 
\begin{equation*}
\theta_T \in \argmax_{\theta\in\Theta} \mathbb{E}_{\bar\lambda}\left[\mathbb{I}_{S(X)\geq \alpha}\ln P(X;\theta)\right],    
\end{equation*}
which implies that $\theta_T$ is a CE-Optimal solution. Therefore, Algorithm~\ref{alg:alg1} converges to a CE-Optimal solution in a finite number of iterations.
\end{proof}

\subsection{Proof of Theorem~\ref{thm:main_theorem}} \label{sec:proof}
\begin{proof}[Proof of Theorem \ref{thm:main_theorem}]
Let $p(x;\bar\lambda)$ be the initial sampling distribution over the initial set of architectures $\mathcal X_0$. We can consider $\gamma_0$ to be the $(1 - \rho)$-quantile under $p(x;\bar\lambda)$. Then, we can consider $\rho$, $\delta > 0$ and $\bar\lambda$ to be the inputs to Algorithm \ref{alg:alg1} of Section \ref{sec:crossEntropy}. Furthermore, by having $\gamma_t >= \min(\alpha, \gamma_{t-1} + \delta)$, we can view $\gamma_t$ as the $(1-\rho_t)$-quantile under the sampling density $p(x;\theta_t)$ provided by the generator $G(.;\theta_t)$ at the $t$-th iteration, for all $t\in\{1,\ldots,T\}$. 

We now modify Algorithm \ref{alg:alg1} of Section \ref{sec:crossEntropy} by replacing the KL-divergence minimization of Step \ref{step:kl} of Algorithm \ref{alg:alg1} with minimizing the Jensen-Shannon (JS)-divergence to obtain Algorithm \ref{alg:alg2_modified} for JS-divergence minimization for rare event simulation. Note that Step \ref{step:js} of Algorithm \ref{alg:alg2_modified} is equivalent to Step \ref{step:minimax} of Algorithm \ref{alg:alg3}. Furthermore, the distribution of architectures in $\mathcal T$, i.e., the distribution of top architectures in prior iteration in Algorithm \ref{alg:alg3}, corresponds to $q(x,\gamma_{t-1};\bar\lambda)$ in Algorithm~\ref{alg:alg2_modified}. By choosing $k$ such that $k=\rho_t\times|\bigcup_{i=0}^{t-1}\mathcal X_i|$ and the above argument, we note that Algorithm \ref{alg:alg2_modified} and Algorithm \ref{alg:alg3} are equivalent. Therefore, by showing that Algorithm \ref{alg:alg2_modified} converges in a finite number of iterations, we imply that Algorithm \ref{alg:alg3} converges in a finite number of iterations.

Using the convergence result of Algorithm \ref{alg:alg1}, we now show that Algorithm \ref{alg:alg2_modified} converges in a finite number of iterations.

Consider the JS divergence between $q(x, \gamma_{t-1};\bar\lambda) = c^{-1}\mathbb{I}_{S(x)\geq\gamma_{t-1}}p(x;\bar\lambda)$ and $p(x;\theta)$, where $c = \int \mathbb{I}_{S(x)\geq\gamma_{t-1}}p(x;\bar\lambda)dx$, $\gamma_{t-1} = \min (\alpha, \gamma(\theta_{t-1}, \rho_{t-1}))$, which is
\begin{eqnarray}\label{eq:JS1}
  &&JS(q(x, \gamma_{t-1};\bar\lambda)||p(x;\theta)) = \nonumber\\ &&\frac{1}{2}KL (c^{-1}p(x;\bar\lambda)\mathbb{I}_{S(X)\geq\gamma_{t-1}} \parallel \nonumber\\ &&\frac{1}{2}\left(c^{-1}p(x;\bar\lambda)\mathbb{I}_{S(X)\geq\gamma_{t-1}} + p(x;\theta))\right) + \nonumber\\
  &&\frac{1}{2}KL (p(x;\theta) \parallel \frac{1}{2} (c^{-1}p(x;\bar\lambda)\mathbb{I}_{S(X)\geq\gamma_{t-1}} + \nonumber\\
  && p(x;\theta))).
\end{eqnarray}
%Motivated by the well-known fact that the symmetric $JS$ divergence is more robust than the asymmetric $KL$ divergence \cite{goodfellow2016nips,arjovsky2017towards}, we modify Algorithm~\ref{alg:alg1} by minimizing the JS-divergence in (\ref{eq:JS1}) instead of the KL-divergence of (\ref{eq:CE_Opt1}) and (\ref{eq:opt}) to obtain Algorithm~\ref{alg:alg2_modified}, which naturally leads to the proposed GA-NAS scheme. In fact, the differences between GANs (relying on JS-divergence) and Variational Auto Encoders (VAEs, which rely on KL divergence), in terms of comparing objectives, are comprehensively discussed on page 14 of \cite{goodfellow2016nips} as well as page 2 of \cite{arjovsky2017towards}. Let the real and generated data are from the probability distribution $P_{real}(.)$ and $P_{gen}(.)$, respectively. According to the discussion in \cite{arjovsky2017towards}, asymmetric KL-divergence fails to work properly in two extreme cases. In the case $P_{real}(x) > P_{gen}(x)$, $P_{real}(x) > 0$, and $P_{gen}(x)$ goes to zero, the generator “does not cover parts of data”, and in the case $P_{gen}(x) > P_{real}(x)$ , $P_{gen}(x) > 0$, and $P_{real}(x)$  goes to zero, the generator generates “fake looking samples”. \red{Therefore, it is a general belief in the literature that the symmetry of JS-divergence with respect to $P_{real}(x)$ and $P_{gen}(x)$ causes GANs to generate samples of better quality than VAEs.}
We say $\theta^*$ is a \emph{JS-optimal solution}, if 
\begin{align}\label{eq:JS_Opt}
    \theta^* \in \arg\min_{\theta\in\Theta} JS(q(x, \gamma_{t-1};\bar\lambda)||p(x;\theta)).
\end{align}

\begin{algorithm}[h]
	\caption{JS-divergence minimization for rare-event estimation} \label{alg:alg2_modified}
	\begin{algorithmic}[1]
		\State\textbf{Input:} Level parameter $\alpha > 0 $, some fixed parameter $\delta > 0$, and initial 
	    parameters $\bar\lambda$ and $\rho$.
	    \State $t\leftarrow 1$, $\rho_0\leftarrow\rho$, $\theta_0\leftarrow\bar\lambda$
	    \While {$\gamma(\theta_{t-1}, \rho_{t-1}) < \alpha$}
	        \State $\theta_t \in \arg\min_{\theta\in\Theta} JS(q(x, \gamma_{t-1};\bar\lambda)||p(x;\theta))$, where $\gamma_{t-1} = \min (\alpha, \gamma(\theta_{t-1}, \rho_{t-1}))$ and 
	        \[
	             q(x, \gamma_{t-1};\bar\lambda) = \frac{\mathbb{I}_{S(x)\geq\gamma_{t-1}}p(x;\bar\lambda)}{\int \mathbb{I}_{S(x)\geq\gamma_{t-1}}p(x;\bar\lambda)dx}. 
	        \]\label{step:js}
	        \State Set $\rho_t$ such that $\gamma(\theta_t, \rho_t) \geq \min(\alpha, \gamma(\theta_{t-1}, \rho_{t-1}) + \delta)$.
	        \State $t\leftarrow t+1$.
	   \EndWhile
	\end{algorithmic} 
	%\vspace{-6mm}
\end{algorithm}

% In each iteration of the above algorithm, we fit $\theta_t$ to minimize the JS divergence between the sampling density $p(x;\theta_t)$ and $q(x,\gamma_{t-1};\bar\lambda)$, the density of $X$ conditioned on that $S(X)\ge \gamma_{t-1}$. Since $\gamma_{t-1}$ is typically the $(1-\rho_{t-1})$-quantile of $S(X)$ under the previous sampling density $p(x;\theta_{t-1})$, we are essentially fitting $\theta_t$ to the top candidates that had highest $S(x)$ in the previous iteration.
% Once the final sampling density is obtained, we can use that to generate a solution $x$ such that $S(x) \ge \alpha$ with high probability.
The assumption that $\max_{x\in\mathcal X}S(x) \geq \alpha$, for some $\alpha > 0$, implies the assumption of Theorem~\ref{thm:Thm1}, i.e., $l(\alpha) > 0$. 
% Now, noting that  $\rho_t=\frac{k}{|\bigcup_{i=0}^{t-1}\mathcal X_i|}$, and $\gamma_t$ is the $(1-\rho_t)$-quantile under the sampling density $p(x;\theta_t)$ provided by the generator $G(x;\theta_t)$ at iteration $t$, and assuming that we are able to solve the minimax problem in Step \ref{step:minimax} of Algorithm \ref{alg:alg3} exactly, the distribution of architectures in $\mathcal T$, i.e., the distribution of top architectures in prior iteration, corresponds to $q(x,\gamma_{t-1})$ in Algorithm~\ref{alg:alg2_modified}, and GA-NAS algorithm is essentially an implementation of the JS-divergence rare-event estimation of Algorithm~\ref{alg:alg2_modified}. From \cite{goodfellow2014generative}, this is equivalent to minimizing the JS-divergence between $q(x, \gamma_{t-1};\bar\lambda)$ and $p(x;\theta_{t-1})$ by the end of iteration $t-1$. 
Then, following the same arguments as in the proof of Theorem~\ref{thm:Thm1}, the assertion of Theorem~\ref{thm:Thm1} still holds for the case of JS-divergence minimization of Algorithm \ref{alg:alg2_modified}, and consequently, Algorithm~\ref{alg:alg3} converges to a JS-optimal solution in a finite number of iterations.  
\end{proof}

\subsection{Model Details}
% In this section, we present more details on the GA-NAS algorithm.
% To help facilitate a better understanding, in the remainder of this section, we mainly discuss GA-NAS in the context of cell-based micro search, i.e., we search for a cell architecture $C$ where each node is an operator like Conv3x3, Maxpooling, etc, and edges direct the flow of information between nodes. 
% Under this definition, GA-NAS searches for the node's operator type, the number of nodes in a cell, and the edge connections between nodes.
% We want to re-emphasize that the GA-NAS algorithm is not limited to \textit{only} micro search. 
% In a macro search setting,  the discriminator and generator of GA-NAS can be transferred to search for a single-path network instead of a cell. Each block in the network is treated as a node, and the entire network as a special cell.

% \subsubsection{GA-NAS Algorithm Flow}
% Figure~\ref{fig:ga_nas_flow} provides an illustrative view of the components and steps in the proposed GA-NAS algorithm.

% \begin{figure*}[t]
% 	\centering
% 	\includegraphics[width=5in]{figures/GANAS.png}
% 	\caption{An illustration of the flow of the proposed GA-NAS algorithm.}
% 	\label{fig:ga_nas_flow}
% \end{figure*}

\subsubsection{Pairwise Architecture Discriminator}
\label{sec:sup_d}
Consider the case where we search for a cell architecture $\mathcal C$, the input to $D$ is a pair of cells $(\mathcal C^{true}, \mathcal C')$, 
where $\mathcal C^{true}$ denotes a cell from the current truth data, while  $\mathcal C'$ is a cell from either the truth data or the generator's learned distribution.
We transform each cell in the pair into a graph embedding vector using a shared $k$-GNN model.
The graph embedding vectors of the input cells are then concatenated and passed to an MLP for classification. 
We train $D$ in a supervised fashion and minimize the cross-entropy loss between the target labels and the predictions. 
To create the training data for $D$, we first sample $\mathcal T$ unique cells from the current truth data. 
We create positive training samples by pairing each truth cell against every other truth cell.
We then let the generator generate $|\mathcal T|$ unique and valid cells as fake data and pair each truth cell against all the generated cells.

\subsubsection{Architecture Generator}
A complete architecture is constructed in an auto-regressive fashion. 
We define the state at the $t$-th time step as ${\mathcal C}_t$, which is a partially constructed graph. 
Given the state ${\mathcal C}_{t-1}$, the actor inserts a new node.
The action $a_t$ is composed of an operator type for this new node selected from a predefined operation space and its connections to previous nodes.
We define an episode by a trajectory $\tau$ of length $N$ as the state transitions from ${\mathcal C}_0$ to ${\mathcal C}_{N-1}$, i.e., $\tau=\{{\mathcal C}_0, {\mathcal C}_1, \ldots, {\mathcal C}_{N-1}\}$, where $N$ is an upper bound that limits the number of steps allowed in graph construction.
The actor model follows an Encoder-Decoder framework, where the encoder is a multi-layer $k$-GNN. 
The decoder consists of an MLP and a Gated Recurrent Unit.
The architecture construction terminates when the actor generates a terminal \emph{output} node or the final state ${\mathcal C}_{N-1}$ is reached.
Figure~\ref{fig:g_structure} illustrates an architecture generator used in our NAS-Bench-101 experiments.

\begin{figure*}[t]
	\centering
	\includegraphics[width=5.5in]{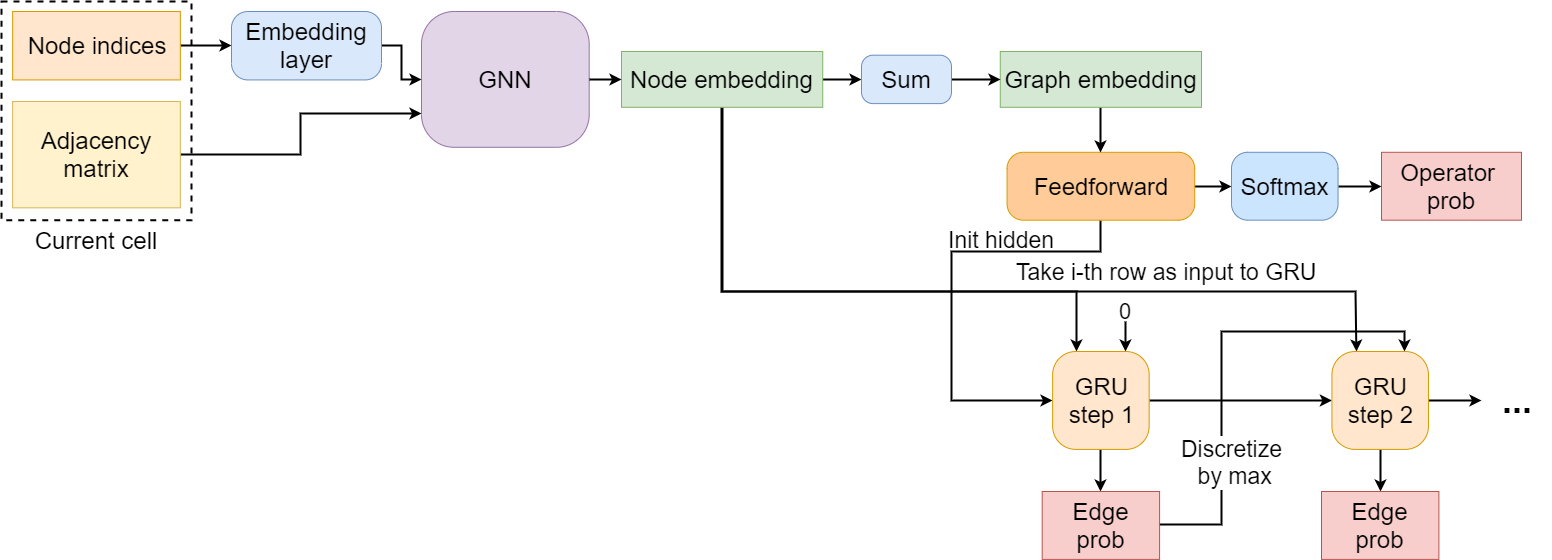}
	\caption{The structure of a GNN architecture generator for cell-based micro search.}
	\label{fig:g_structure}
\end{figure*}

\subsubsection{Training Procedure of the Architecture Generator}

The state transition is captured by a policy $\pi_\theta(a_t|{\mathcal C}_t)$, where the action $a_t$ includes predictions on a new node type and its connections to previous nodes. To learn $\pi_\theta(a_t|{\mathcal C}_t)$ in this discrete action space, we adopt the Proximal Policy Optimization (PPO) %\cite{schulman2017proximal} 
algorithm with generalized advantage estimation.
The actor is trained by maximizing the cumulative expected reward of the trajectory.
For a trajectory $\tau$ with a maximum allowed length of $N$, this objective translates to
\begin{equation}
\begin{split}
     \max E[R(\tau)] = \max E[R_{step}(\tau)] + E[R_{final}(\tau)]\\
     \quad\textrm{s.t. }|\tau| \leq N,
\end{split}
\end{equation}
where $R_{step}$ and $R_{final}$ correspond the per-step reward and the final reward, respectively, which will be described in the following.

In the context of cell-based micro search, there is a step reward $R_{step}$, which is given to the actor after each action, and a final reward $R_{final}$, which is only assigned at the end of a multi-step generation episode.
For $R_{step}$, we assign the generator a step reward of $0$ if $a_t$ is valid. Otherwise, we assign $-0.1$ and terminate the episode immediately.

$R_{final}$ consists of two parts. The first part is a validity score $R_v$. 
For a completed cell $\mathcal C_{gen}$, if it is a valid DAG and contains exactly one output node, and there is at least one path from every other node to the output, then the actor receives a validity score of $R_v(\mathcal C_{gen})=0$. Otherwise, the validity score will be $-0.1$ multiplied by the number of validity violations.
In our search space, we define four possible violations:
1) There is no output node; 2) There is a node with no incoming edges; 3) There is a node with no outgoing edges; 4) There is a node with no incoming or outgoing edges.
The second part of $R_{final}$ is
$R_D(\mathcal C_{gen})$, which represents the probability that the discriminator classifies $\mathcal C_{gen}$ as a cell from the truth data distribution $p_{data}(x)$.
In order to receive $R_D(\mathcal C_{gen})$ from the discriminator, $\mathcal C_{gen}$ must have a validity score of $0$ and $\mathcal C_{gen}$ cannot be one of the current truth cells $\{\mathcal C^j_{true}| j=1,2,...K\}$.
 
Formally, we express the final reward $R_{final}$ for a generated architecture $\mathcal C_{gen}$ as
\begin{equation}
    R_{final} = \begin{cases} R_v(\mathcal C_{gen}) &\mbox{if } I(C_{gen}), \\ 
    R_v(\mathcal C_{gen}) + R_D(\mathcal C_{gen}) & \mbox{otherwise. } \end{cases}
\end{equation}
And $I(C_{gen}) = R_v(\mathcal C_{gen}) < 0 \textrm{ or } \mathcal C_{gen}\in\{\mathcal C^j_{true}| j=1,2,...K\}$.
We compute $R_D(\mathcal C_{gen})$ by conducting pairwise comparisons against the current truth cells, then take the maximum probability that the discriminator will predict a $1$, i.e.,
$R_D=\max_j\textrm{ }P(\mathcal C_{gen} \in p_{data}(x)|\mathcal C_{gen}, \mathcal C^j_{true}; D),\textrm{ for } j=1,2,...K$, as the discriminator reward.

Maintaining the right balance of exploration/exploitation is crucial for a NAS algorithm. 
In GA-NAS, the architecture generator and discriminator provide an efficient way to utilize learned knowledge, i.e., exploitation. 
For exploration, we make sure that the generator always have some uncertainties in its actions by tuning the multiplier for the entropy loss in the PPO learning objective. %\cite{schulman2017proximal}. 
The entropy loss determines the amount of randomness, hence, variations in the generated actions. Increasing its multiplier would increase the impact of the entropy loss term, which results in more exploration. 
In our experiments, we have tuned this multiplier extensively and found that a value of $0.1$ works well for the tested search spaces.

Last but not least, it is worth mentioning that the above formulation of the reward function also works for the single-path macro search scenario, such as EfficientNet and ProxylessNAS search spaces, in which we just need to modify $R_{step}$ and $R_D(\mathcal C_{gen})$ according to the definitions of the new search space.

\subsection{Experimentation}\label{sup:exp}
In this section, we present ablation studies and the results of additional experiments. We also provide more details on our experimental setup.
We implement GA-NAS using PyTorch 1.3. 
We also use the PyTorch Geometric library %\cite{Fey/Lenssen/2019} 
for the implementation of $k$-GNN.

\subsubsection{More Ablation Studies}\label{sec:ablation}
We report the results of ablation studies to show the effect of different components of GA-NAS on its performance.

\textbf{Standard versus Pairwise Discriminator}: In order to investigate the effect of pairwise discriminator in GA-NAS explained in section \ref{sec:algorithm_generator}, we perform the same experiments as in Table~\ref{table:nb101_avg_acc}, with the standard discriminator where each architecture $\in\mathcal T$ is compared against a generated architecture $\in\mathcal F$. The results presented in Table~\ref{table:nb101_disc_ablation} indicate that using pairwise discriminator leads to a better performance compared to using standard discriminator. % to find the best architecture.}    
\begin{table*}[h]
   \centering
   \small
	\begin{tabular}{|l|c|c|c|} \hline
          \textbf{Algorithm}    &  \textbf{Mean Acc} & \textbf{Mean Rank} &\textbf{Average \#Q}   \\ \hline
		\textbf{GA-NAS-Setup1 with pairwise discriminator} &\textbf{94.221}$\pm$ 4.45e-5 & 2.9 & \textbf{647.5} $\pm$ 433.43    \\
		GA-NAS-Setup1 without pairwise discriminator   &94.22$\pm$ 0.0 & 3 & 771.8 $\pm$ 427.51    \\
        \textbf{GA-NAS-Setup2 with pairwise discriminator}  &\textbf{94.227}$\pm$ 7.43e-5 & 2.5 & 1561.8 $\pm$ 802.13  \\
        GA-NAS-Setup2 without pairwise discriminator  &94.22$\pm$ 0.0 & 3 & 897 $\pm$ 465.20  \\
		\hline
	\end{tabular}
	%\vspace{-3mm}
	\caption{\label{table:nb101_disc_ablation}  The mean accuracy and rank, and average number of queries over 10 runs.}% to find the best architecture.}   
	%\vspace{-2mm}
\end{table*}

\textbf{Uniform versus linear sample size increase with fixed number of evaluation budget}: Once the generator in Algorithm~\ref{alg:alg3} is trained, we sample $|\mathcal X_t|$ architectures $\mathcal X_t$. Intuitively, as the algorithm progresses, $G$ becomes more and more accurate, thus, increasing the size of $\mathcal X_t$ over the iterations should prove advantageous. We provide some evidence that this is the case. 
%We also investigate the effect of linear increase in the number of generated cell architectures in each iteration of the algorithm. 
More precisely, we perform the same experiments as in Table~\ref{table:nb101_avg_acc}; however, keeping the total number of generated cell architectures during the 10 iterations of the algorithm the same as that in the setup of Table~\ref{table:nb101_avg_acc}, we generate a constant number of cell architectures of $225$ and $450$ in each iteration of the algorithm in setup1 and setup2, respectively. The results presented in Table~\ref{table:nb101_inc_ablation} indicate that a linear increase in the size of generated cell architectures leads to a better performance.
\begin{table*}[h]
   \centering
   \small
	\begin{tabular}{|l|c|c|c|} \hline
          \textbf{Algorithm}    &  \textbf{Mean Acc} & \textbf{Mean Rank} &\textbf{Average \#Q}   \\ \hline
		\textbf{GA-NAS-Setup1 ($|\mathcal X_t|=|\mathcal X_{t-1}|+50,~\forall t \geq 2$}) &\textbf{94.221} $\pm$ 4.45e-5 & \textbf{2.9} & \textbf{647.5} $\pm$ 433.43    \\
		GA-NAS-Setup1 ($|\mathcal X_t|=225,~\forall t \geq 2$) &94.21$\pm$ 0.000262 & 3.4 & 987.7 $\pm$ 394.79    \\
        \textbf{GA-NAS-Setup2 ($|\mathcal X_t|=|\mathcal X_{t-1}|+50,~\forall t \geq 2$})  &\textbf{94.227}$\pm$ 7.43e-5 & \textbf{2.5} & 1561.8 $\pm$ 802.13  \\
        GA-NAS-Setup2 ($|\mathcal X_t|=450,~\forall t \geq 2$)  &94.22$\pm$ 0.0 & 3 & 1127.6 $\pm$ 363.75  \\
		\hline
	\end{tabular}
	%\vspace{-3mm}
	\caption{\label{table:nb101_inc_ablation}  The mean accuracy and rank, and average number of queries over 10 runs.}% to find the best architecture.}   
	%\vspace{-2mm}
\end{table*}

\subsubsection{Pareto Front Search Results on NAS-Bench-101}
In addition to constrained search, we search through Pareto frontiers to further illustrate our algorithm's ability to learn any given truth set.
We consider test accuracy vs. normalized training time and found that the truth Pareto front of NAS-Bench-101 contains 41 cells.
To reduce variance, we always initialize with the worst 50\% of cells in terms of accuracy and training time, which amounts to 82,329 cells.
We modify GA-NAS to take a Pareto neighborhood of size 4 in each iteration, which is defined as iteratively removing the cells in the current Pareto front and finding a new front using the remaining cells, until we have collected cells from 4 Pareto fronts.  
We run GA-NAS for $10$ iterations and compare with a random search baseline. 
GA-NAS issued 2,869 unique queries (\#Q) to the benchmark, so we set the \#Q for random search to 3,000 for a fair comparison. 
Figure~\ref{fig:nb101_pareto} showcases the effectiveness of GA-NAS in uncovering the Pareto front.
While random search also finds a Pareto front that is close to the truth, a small gap is still visible.
In comparison, GA-NAS discovers a better Pareto front with a smaller number of queries.
GA-NAS finds 10 of the 41 cells on the truth Pareto front, and random search only finds 2.

\begin{figure*}
\centering
\begin{subfigure}[b]{0.48\textwidth}
  \centering
  \includegraphics[width=\linewidth]{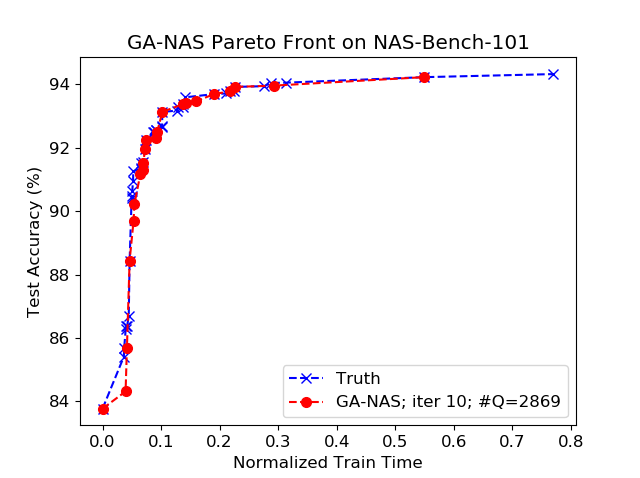}
  % \caption{GA-NAS}
  \label{fig:nb101_pareto_ga_nas}
\end{subfigure}
\begin{subfigure}[b]{0.48\textwidth}
  \centering
  \includegraphics[width=\linewidth]{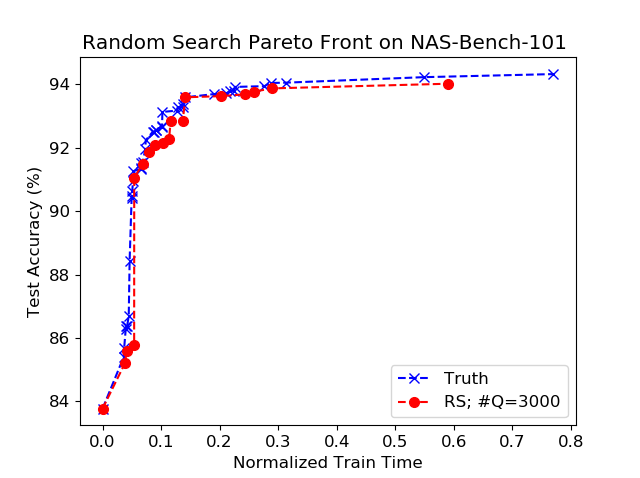}
  % \caption{Random Search (RS)}
  \label{fig:nb101_pareto_rs}
\end{subfigure}
\caption{Pareto Front search on NAS-Bench-101. Observe that GA-NAS nearly recovers the truth Pareto Front while Random Search (RS) only discovers one close to the truth. \#Q represents the total number of unique queries made to the benchmark. Each marker represents a cell on the Pareto front.}
\label{fig:nb101_pareto}
\end{figure*}

\subsubsection{Hyper-parameter Setup}
Table~\ref{table:h_params} reports the key hyper-parameters used in our NAS-Bench-101, 201 and 301 experiments, which includes (1) best-accuracy search by querying the true accuracy (Acc), (2) best-accuracy search using the supernet (Acc-WS), (3) constrained best-accuracy search (Acc-Cons), and (4) Pareto front search (Pareto).

Here, we include a brief description for some parameters in Table~\ref{table:h_params}.
\begin{itemize}
  \item \textit{\# Eval base ($|\mathcal X_1|$)} is the number of unique, valid cells to be generated by $G$ after the first iteration of $D$ and $G$ training, i.e., last step of an iteration of Algorithm.~\ref{alg:alg3}.
  \item \textit{\# Eval inc ($|\mathcal X_t| - |\mathcal X_{t-1}|$)} is the increment to the number of generated cells after completing an iteration. From the CE method and importance sampling described in Sec.~\ref{sec:crossEntropy}, the early generator distribution is not close enough to the well-performing cell architecture. To lower the number of queries to the benchmark without sacrificing the performance, we propose an incremental increase in the number of generated cell architectures at each iteration.
  \item \textit{Init method} describes how to choose the initial set of cells, from which we create the initial truth set. For most experiments, we randomly choose $|\mathcal{X}_0|$ number of cells as the initial set. For the Pareto front search, we initialize with cells that rank in the lower 50\% in terms of both test accuracy and training time (which constitutes 82,329 cells in NAS-Bench-101). 
  \item \textit{Truth set size ($\mathcal{T}$)} controls the number of truth cells for training $D$. For best-accuracy searches, we take the top most accurate cells found so far. For Pareto front searches, we iteratively collect the cells on the current Pareto front, then remove them from the current pool, then find a new Pareto front, until we visit the desired number of Pareto fronts. 
\end{itemize}
For the complete set of hyper-parameters, please check out our code.

\begin{table*}[t]
	\centering
	\small
	\scalebox{0.99}{
	\begin{threeparttable}[hb]
	\caption{Key hyper-parameters used by our NAS-Bench-101, 201 and 301 experiments. Among multiple runs of an experiment, the same hyper-parameters are used and only the random seed differs.}
	\label{table:h_params}
	\begin{tabular}{|l||c|c|c|c||c||c|} \hline
	    & \multicolumn{4}{c||}{\textbf{NAS-Bench-101}} & \multicolumn{1}{c||}{\textbf{NAS-Bench-201}} & \multicolumn{1}{c|}{\textbf{NAS-Bench-301}}\\ \hline
        \textbf{Parameter} & \textbf{Acc (2 setups)} & \textbf{Acc-WS} & \textbf{Acc-Cons} & \textbf{Pareto} & \textbf{Acc (3 datasets)} & \textbf{Acc}\\ \hline
        $G$ optimizer                     & Adam & Adam & Adam & Adam & Adam & Adam \\
        $G$ learn rate                    & 0.0001 & 0.0001 & 0.0001 & 0.0001 & 0.0001 & 0.0001              \\
        $D$ optimizer                     & Adam & Adam & Adam & Adam & Adam & Adam                    \\
        $D$ learn rate                    & 0.001 & 0.001 & 0.001 & 0.001 & 0.001  & 0.001                \\
        \# GNN layers                     & 2 & 2 & 2 & 2 & 2 & 2                              \\
		Iterations ($T$)                  & 10 & 5 & 5 & 10 & 5 & 10                                     \\
		\# Eval base ($|\mathcal{X}_1|$)  & 100 & 100 & 200 & 500 & 60 & 50                             \\
		\# Eval inc ($|\mathcal{X}_t| - |\mathcal{X}_{t-1}|$) & 50, 100 & 100 & 200 & 100 & 10 & 50                              \\
		Init method                       & Random & Random & Random & Worst 50\% & Random & Random                \\
		Init size ($|\mathcal{X}_0|$)     & 50; 100 & 100 & 100 & 82329 & 60 & 100                            \\
		Truth set size ($|\mathcal{T}|$)  & 25; 50  & 50  & 50  & 4 fronts & 40 & 50                          \\
		\hline
	\end{tabular}
	\end{threeparttable}
	}
	%\vspace{-4mm}
\end{table*}

\subsubsection{Supernet Training and Usage}
\label{sup:exp_supernet}
To train a supernet for NAS-Bench-101, we first set up a macro-network in the same way as the evaluation network of NAS-Bench-101. 
Our supernet has 3 stacks with 3 supercells per stack. 
A downsampling layer consisting of a Maxpooling 2-by-2 and a Convolution 1-by-1 is inserted after the first and second stack to halve the input width/height and double the channels.
Each supercell contains 5 searchable nodes.
In a NAS-Bench-101 cell, the output node performs concatenation of the input features. However, this is not easy to handle in a supercell. 
Therefore, we replace the concatenation with summation and do not split channels between different nodes in a cell.

We train the supernet on 40,000 randomly sampled CIFAR-10 training data and leave the other 10,000 as validation data.
We set an initial channel size of 64 and a training batch size of 128.
We adopt a uniformly random strategy for training, i.e., for every batch of training data, we first uniformly sample an edge topology, then, we uniformly sample one operator type for every searchable node in the topology.
Following this strategy, we train for 500 epochs using the Adam optimizer and an initial learning rate of 0.001.
We change the learning rate to 0.0001 when the training accuracy do not improve for 50 consecutive epochs.

During search, after we acquire the set of cells ($\mathcal{X}$) to evaluate, we first fine-tune the current supernet for another 50 epochs. 
The main difference here compared to training a supernet from scratch is that for a batch of training data, we randomly sample a cell from $\mathcal{X}$ instead of from the complete search space.
We use the Adam optimizer with a learning rate of 0.0001.
Then, we get the accuracy of every cell in $\mathcal{X}$ by testing on the validation data and by inheriting the corresponding weights of the supernet.

\subsubsection{ResNet and Inception Cells in NAS-Bench-101}
Figure~\ref{fig:res_net_inception_cells} illustrates the structures of the ResNet and Inception cells used in our constrained best-acc search.
Note that both cells are taken from the NAS-Bench-101 database as is, there might be small differences in their structures compared to the original definitions.
Table~\ref{table:nb101_cons_acc_search} reports that the ResNet cell has a lot more weights than the inception cell, even though it contains fewer operators. This is because NAS-Bench-101 performs channel splitting so that each operator in a branched path will have a much fewer number of trainable weights.

We then present the two best cell structures found by GA-NAS that are better than the ResNet and Inception cells, respectively, in Figure~\ref{fig:res_net_inception_repalce_cells}. 
Both cells are also in the NAS-Bench-101 database.
Observe that both cells contain multiple branches coming from the input node.

\begin{figure*}
\centering
\begin{subfigure}[b]{0.2\textwidth}
  \centering
  \includegraphics[width=\linewidth]{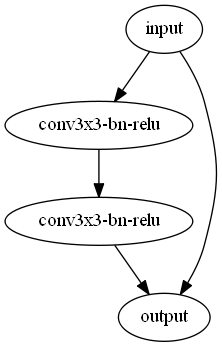}
  \caption{ResNet cell}
  \label{fig:resnet_cell}
\end{subfigure}
\begin{subfigure}[b]{0.75\textwidth}
  \centering
  \includegraphics[width=0.8\linewidth]{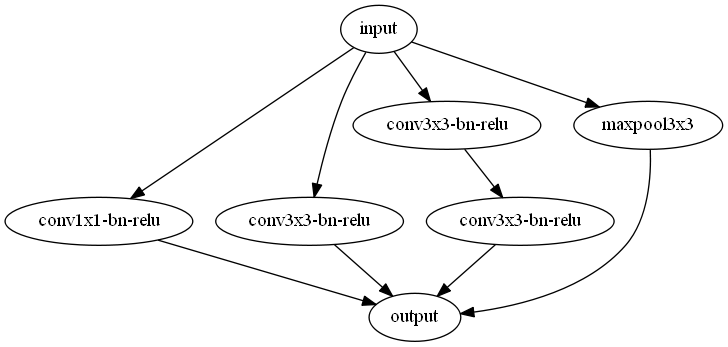}
  \caption{Inception cell}
  \label{fig:inception_ell}
\end{subfigure}
\caption{Structures of the ResNet and Inception cells, which are considered hand-crafted architectures in our constrained best-accuracy search.}
\label{fig:res_net_inception_cells}
\end{figure*}

\begin{figure*}
\centering
\begin{subfigure}[b]{0.5\textwidth}
  \centering
  \includegraphics[width=\linewidth]{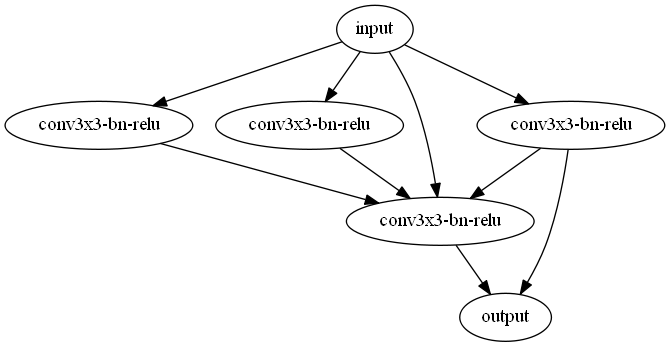}
  \caption{ResNet cell replacement.}
  \label{fig:resnet_replace}
\end{subfigure}
\begin{subfigure}[b]{0.45\textwidth}
  \centering
  \includegraphics[width=\linewidth]{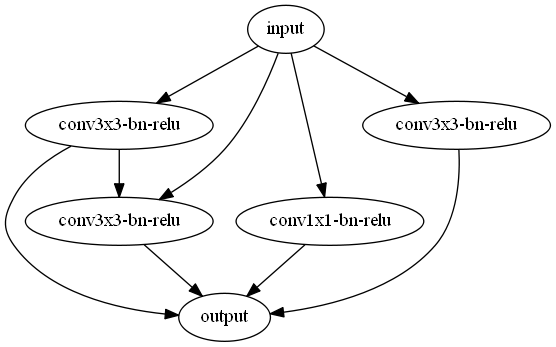}
  \caption{Inception cell replacement.}
  \label{fig:inception_replace}
\end{subfigure}
\caption{The structures of the best two cells found by GA-NAS that are better than the ResNet and Inception cells in terms of test accuracy, training time and number of weights.}
\label{fig:res_net_inception_repalce_cells}
\end{figure*}

\subsubsection{More on NAS-Bench-101 and NAS-Bench-201}
We summarize essential statistical information about NAS-Bench-101 and NAS-Bench-201 in Table~\ref{table:nas_bench_info}. 
The purpose is to establish a common ground for comparisons.
Future NAS works who wish to compare to our experimental results directly can check Table~\ref{table:nas_bench_info} to ensure a matching benchmark setting.
There are 3 candidate operators for NAS-Bench-101, (1) A sequence of convolution 3-by-3, batch normalization (BN), and ReLU activation (Conv3$\times$3-BN-ReLU), (2) Conv1$\times$1-BN-ReLU, (3) Maxpool3$\times$3.
NAS-Bench-201 defines 5 operator choices: (1) zeroize, (2) skip connection, (3) ReLU-Conv1$\times$1-BN, (4) ReLU-Conv3$\times$3-BN, (5) Averagepool3$\times$3. 
We want the re-emphasize that for each cell in NAS-Bench-101, we take the average of the final test accuracy at epoch 108 over 3 runs as its true test accuracy.
For NAS-Bench-201, we take the labeled accuracy on the test sets.
Since there is a single edge topology for all cells in NAS-Bench-201, there is no need to predict the edge connections; hence, we remove the GRU in the decoder and only predict the node types of 6 searchable nodes.

\begin{table*}[t]
	\centering
	\small
	\scalebox{0.99}{
	\begin{threeparttable}[hb]
	\begin{tabular}{|l||c||c|c|c|} \hline
                & \textbf{NAS-Bench-101} & \multicolumn{3}{c|}{\textbf{NAS-Bench-201}}  \\ \hline
        Dataset & CIFAR-10 & CIFAR-10 & CIFAR-100 & ImageNet-16-120 \\ \hline
        \# Cells & 423,624 & 15,625 & 15,625 & 15,625             \\
        Highest test acc & 94.32 & 94.37 & 73.51 & 47.31           \\
        Lowest test acc & 9.98 & 10.0 & 1.0 & 0.83                  \\
        Mean test acc & 89.68 & 87.06 & 61.39 & 33.57                \\
        \# Operator choices & 3 & 5 & 5 & 5 \\
        \# Searchable nodes per cell & 5 & 6 & 6 & 6 \\
        Max \# edges per cell & 9 & - & - & - \\
		\hline
	\end{tabular}
	\end{threeparttable}
	}
	\caption{Key information about the NAS-Bench-101 and NAS-Bench-201 benchmarks used in our experiments.}
	\label{table:nas_bench_info}
	%\vspace{-4mm}
\end{table*}

\subsubsection{Clarification on Cell Ranking}
We would like to clarify that all ranking results reported in the paper are based on the \textit{un-rounded} true accuracy.
In addition, if two cells have the same accuracy, we randomly rank one before the other, i.e. no two cells will have the same ranking.

\subsubsection{Conversion from DARTS-like Cells to the NAS-Bench-101 Format}
DARTS-like cells, where an edge represents a searchable operator, and a node represents a feature map that is the sum of multiple edge operations, can be transformed into the format of NAS-Bench-101 cells, where nodes represent searchable operators and edges determine data flow.
For a unique, discrete cell, we assume that each edge in a DARTS-like cell can adopt a single unique operator.
We achieve this transformation by first converting each edge in a DARTS-like cell to a NAS-Bench-101 node. 
Next, we construct the dependency between NAS-Bench-101 nodes from the DARTS nodes, which enables us to complete the edge topology in a new NAS-Bench-101 cell.
Figure~\ref{fig:cell_transformation} shows a DARTS-like cell defined by NAS-Bench-201 and the transformed NAS-Bench-101 cell.
This transformation is a necessary first step to make GA-NAS compatible with NAS-Bench-201 and NAS-Bench-301.
Notice that every DARTS-like cell has the same edge topology in the NAS-Bench-101 format, which alleviates the need for a dedicated edge predictor in the decoder of $G$.

\begin{figure*}[t]
	\centering
	\includegraphics[width=4in]{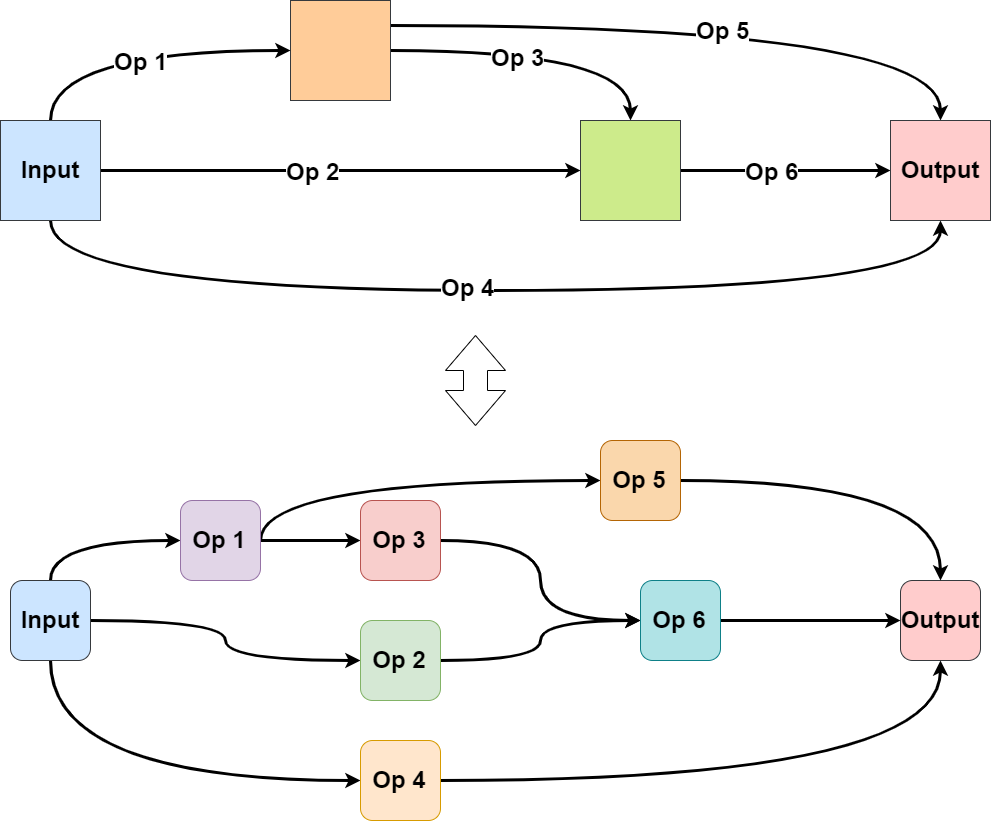}
	\caption{Illustration of how to transform a DARTS-like cell (top) in NAS-Bench-201 to the NAS-Bench-101 format (bottom).}
	\label{fig:cell_transformation}
\end{figure*}

\subsubsection{EfficientNet and ProxylessNAS Search Space}
For EfficientNet experiment, we take the EfficientNet-B0 network structure as the backbone and define 7 searchable locations, as indicated by the TBS symbol in Table~\ref{table:eb0_search_space}. 
We run GA-NAS to select a type of mobile inverted bottleneck MBConv %\cite{sandler2018mobilenetv2} 
block. We search for different expansion ratios \{1, 3 ,6\} and kernel size \{3, 5\} combinations, which results in 6 candidate MBConv blocks per TBS location.

We conduct 2 GA-NAS searches with different performance estimation methods. 
In setup 1 we estimate the performance of a candidate network by training it on CIFAR-10 for 20 epochs then compute the test accuracy. 
In setup 2 we first train a weight-sharing Supernet that has the same structure as the backbone in Table~\ref{table:eb0_search_space}, for 500 epochs, and using a ImageNet-224-120 dataset that is subsampled the same way as NAS-Bench-201. The estimated performance in this case is the validation accuracy a candidate network could achieve on ImageNet-224-120 by inheriting weights from the Supernet. 
For setup 2, training the supernet takes about 20 GPU days on Tesla V100 GPUs, and search takes another GPU day, making a total of 21 GPU days. 

\begin{table}[h]
   \centering
   \small
	\begin{tabular}{|c|c|c|c|} 
	\toprule
    \textbf{Operator} & \textbf{Resolution} & \textbf{\#C} & \textbf{\#Layers} \\
    \midrule
	 Conv3x3 & $ 224 \times 224 $ & 32 & 1 \\
	 TBS & $ 112 \times 112 $ & 16 & 1 \\
	 TBS & $ 112 \times 112 $ & 24 & 2 \\
	 TBS & $ 56 \times 56 $ & 40 & 2 \\
	 TBS & $ 28 \times 28 $ & 80 & 3 \\
	 TBS & $ 14 \times 14 $ & 112 & 3 \\
	 TBS & $ 14 \times 14 $ & 192 & 4 \\
	 TBS & $ 7 \times 7 $ & 320 & 1 \\
	 Conv1x1 \& Pooling \& FC & $ 7 \times 7 $ & 1280 & 1 \\
	\bottomrule
	\end{tabular}
	\caption{ \label{table:eb0_search_space} Macro search backbone network for our EfficientNet experiment. TBS denotes a cell position to be searched, which can be a type of MBConv block. Output denotes a Conv1x1 \& Pooling \& Fully-connected layer.}  
	%\vspace{-4mm}
\end{table}
Figure~\ref{fig:eb0_archs} presents a visualization on the three single-path networks found by GA-NAS on the EfficientNet search space. Compared to EfficientNet-B0, GA-NAS-ENet-1 significantly reduces the number of trainable weights while maintaining an acceptable accuracy.
GA-NAS-ENet-2 improves the accuracy while also reducing the model size. 
GA-NAS-ENet-3 improves the accuracy further.

\begin{figure*}[t]
	\centering
	\includegraphics[width=4in]{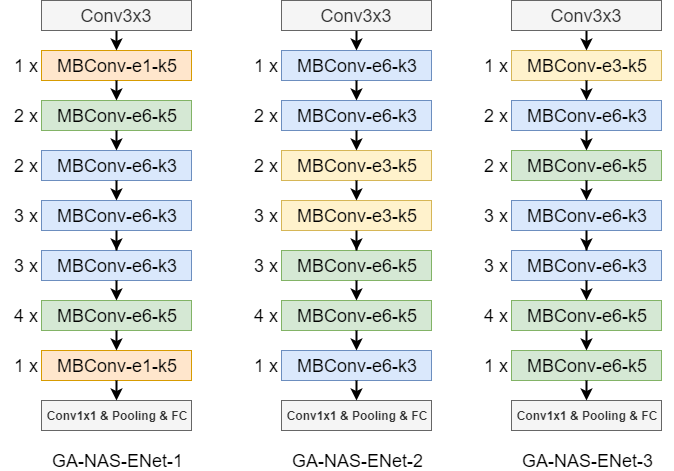}
	\caption{Structures of the single-path networks found by GA-NAS on the EfficientNet search space. In each MBConv block, $e$ denotes expansion ratio and $k$ stands for kernel size.}
	\label{fig:eb0_archs}
\end{figure*}

For ProxylessNAS experiment, we take the ProxylessNAS network structure as the backbone and define 21 searchable locations, as indicated by the TBS symbol in Table~\ref{table:proxylessnas_search_space}. 
We run GA-NAS to search for MBConv blocks with different expansion ratios \{3 ,6\} and kernel size \{3, 5, 7\} combinations, which results in 6 candidate MBConv blocks per TBS location.

We train a supernet on ImageNet %\cite{russakovsky2015imagenet}
for 160 epochs (which is the same number of epochs performed by ProxylessNAS weight-sharing search ~\cite{cai2018proxylessnas}) for around 20 GPU days, and conduct an unconstrained search using GA-NAS for around 38 hours on 8 Tesla V100 GPUs in the search space of ProxylessNAS~\cite{cai2018proxylessnas}, a major portion out of which, i.e., 29 hours is spent on querying the supernet for architecture performance. Figure~\ref{fig:pn_archs} presents a visualization of the best architecture found by GA-NAS-ProxylessNAS with better top-1 accuracy and a comparable the number of trainable weights compared to ProxylessNAS-GPU.
\begin{figure*}[t]
	\centering
	\includegraphics[width=5.5in]{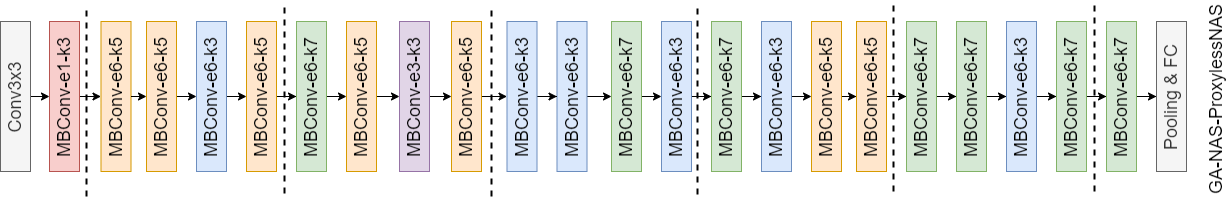}
  	\caption{Structures of the single-path networks found by GA-NAS on the ProxylessNAS search space. }
	\label{fig:pn_archs}
\end{figure*}

\newpage
\begin{table*}[t]
   \centering
   \small
	\begin{tabular}{|c|c|c|c|} 
	\toprule
    \textbf{Operator} & \textbf{Resolution} & \textbf{\#C} & \textbf{Identity} \\
    \midrule
	 Conv3x3 & $ 224 \times 224 $ & 32 & No \\
	 MBConv-e1-k3 & $ 112 \times 112 $ & 16 & No \\
	 TBS & $ 112 \times 112 $ & 24 & No \\
	 TBS & $ 56 \times 56 $ & 24 & Yes \\
	 TBS & $ 56 \times 56 $ & 24 & Yes \\
	 TBS & $ 56 \times 56 $ & 24 & Yes \\
	 TBS & $ 56 \times 56 $ & 40 & No \\
	 TBS & $ 28 \times 28 $ & 40 & Yes \\
	 TBS & $ 28 \times 28 $ & 40 & Yes \\
	 TBS & $ 28 \times 28 $ & 40 & Yes \\
	 TBS & $ 28 \times 28 $ & 80 & No \\
	 TBS & $ 14 \times 14 $ & 80 & Yes \\
	 TBS & $ 14 \times 14 $ & 80 & Yes \\
	 TBS & $ 14 \times 14 $ & 80 & Yes \\
	 TBS & $ 14 \times 14 $ & 96 & No \\
	 TBS & $ 14 \times 14 $ & 96 & Yes \\
	 TBS & $ 14 \times 14 $ & 96 & Yes \\
	 TBS & $ 14 \times 14 $ & 96 & Yes \\
	 TBS & $ 14 \times 14 $ & 192 & No \\
	 TBS & $ 7 \times 7 $ & 192 & Yes \\
	 TBS & $ 7 \times 7 $ & 192 & Yes \\
	 TBS & $ 7 \times 7 $ & 192 & Yes \\
	 TBS & $ 7 \times 7 $ & 320 & Yes \\
	 Avg. Pooling & $ 7 \times 7 $ & 1280 & 1 \\
	 FC & $ 1 \times 1 $ & 1000 & 1 \\
	\bottomrule
	\end{tabular}
	\caption{ \label{table:proxylessnas_search_space} Macro search backbone network for our ProxyLessNAS experiment. TBS denotes a cell position to be searched, which can be a type of MBConv block. Identity denotes if an identity shortcut is enabled.}  
	%\vspace{-4mm}
\end{table*}

\end{document}